\documentclass{article}

\usepackage[utf8]{inputenc} 
\usepackage[T1]{fontenc}    
\usepackage[colorlinks, linkcolor=red, anchorcolor=blue, citecolor=blue]{hyperref}
\usepackage{url}            
\usepackage{booktabs}       
\usepackage{amsfonts}       
\usepackage{nicefrac}       
\usepackage{microtype}      
\usepackage{xcolor,colortbl} 
\usepackage{wrapfig}
\usepackage{caption}
\usepackage{enumitem}
\usepackage{tabularx}
\usepackage{setspace}
\usepackage{graphicx}
\usepackage{subfigure}
\usepackage{amsmath}
\usepackage{amssymb}
\usepackage{mathtools}
\usepackage{amsthm}
\usepackage[capitalize,noabbrev]{cleveref}
\usepackage{todonotes}

\usepackage{mylatexstyle}

\definecolor{LightCyan}{rgb}{0.8, 0.9, 1}
\definecolor{ao}{rgb}{0.0, 0.5, 0.0}

\newcommand{\method}{\texttt{SPIN}}

\allowdisplaybreaks

\def\CC{\textcolor{black}}

\usepackage[accepted]{icml2024} 

\usepackage[compact]{titlesec}
\usepackage{sidecap}
\titlespacing*{\section}{0pt}{*0.5}{*0.5}
\titlespacing*{\subsection}{0pt}{*0.5}{*0.5}
\titlespacing*{\subsubsection}{0pt}{*0.5}{*0.5}
\begin{document}

\twocolumn[
\icmltitle{Self-Play Fine-Tuning Converts Weak Language Models to Strong Language Models}



\icmlsetsymbol{equal}{*}

\begin{icmlauthorlist}
\icmlauthor{Zixiang Chen}{equal,ucla}
\icmlauthor{Yihe Deng}{equal,ucla}
\icmlauthor{Huizhuo Yuan}{equal,ucla}
\icmlauthor{Kaixuan Ji}{ucla}
\icmlauthor{Quanquan Gu}{ucla}
\end{icmlauthorlist}

\icmlaffiliation{ucla}{Department of Computer Science, University of California, Los Angeles, CA 90095, USA}

\icmlcorrespondingauthor{Quanquan Gu}{qgu@cs.ucla.edu}

\icmlkeywords{Machine Learning, ICML}

\vskip 0.3in
]



\printAffiliationsAndNotice{\icmlEqualContribution} 

\begin{abstract}%
 Harnessing the power of human-annotated data through Supervised Fine-Tuning (SFT) is pivotal for advancing  Large Language Models (LLMs). In this paper, we delve into the prospect of growing a strong LLM out of a weak one without the need for acquiring additional human-annotated data. We propose a new fine-tuning method called Self-Play fIne-tuNing ($\method$), which starts from a supervised fine-tuned model.
 At the heart of $\method$ lies a self-play mechanism, where the LLM refines its capability by playing against instances of itself. More specifically, the LLM generates its own training data from its previous iterations, refining its policy by discerning these self-generated responses from those obtained from human-annotated data. Our method progressively elevates the LLM from a nascent model to a formidable one, unlocking the full potential of human-annotated demonstration data for SFT. 
 Theoretically, we prove that the global optimum to the training objective function of our method is achieved only when the LLM policy aligns with the target data distribution. Empirically, we evaluate our method on several benchmark datasets including the HuggingFace Open LLM Leaderboard, MT-Bench, and datasets from Big-Bench. Our results show that $\method$ can significantly improve the LLM's performance across a variety of benchmarks and even outperform models trained through direct preference optimization (DPO) supplemented with extra GPT-4 preference data. This sheds light on the promise of self-play, enabling the achievement of human-level performance in LLMs without the need for expert opponents. \CC{Codes
are available at \url{https://github.com/uclaml/SPIN}.} 

\end{abstract}

\section{Introduction}

Large Language Models (LLMs) have began a  groundbreaking era in artificial general intelligence (AGI), demonstrating extraordinary capabilities across a wide range of domains that require intricate reasoning and specialized knowledge. 
These models excel in areas such as mathematical reasoning/problem solving~\citep{cobbe2021training, wei2022chain,lewkowycz2022solving}, code generation/programming~\citep{chen2021evaluating, austin2021program, li2022competition}, text generation~\citep{bubeck2023sparks, anil2023palm, touvron2023llama}, summarization and creative writing, among others. 
A significant advancement in LLMs is the post-pre-training alignment with the more desirable behaviors~\citep{mishra2021cross, victor2022multitask, chung2022scaling, thoppilan2022lamda}, a process often reliant on the costly human-annotated data.
Typical alignment methods include Supervised Fine-Tuning (SFT)~\citep {ouyang2022training, tunstall2023zephyr} based on human demonstrations, and Reinforcement Learning from Human Feedback (RLHF)~\citep{christiano2017deep,ziegler2019fine, stiennon2020learning, bai2022training}  based on human preferences.

All the aforementioned alignment methods require a substantial volume of human annotated data. Therefore, there is increasing interest in developing fine-tuning methods that can effectively utilize human data, thereby streamlining the alignment process. This motivates us to study fine-tuning LLMs without the need for additional human-annotated data beyond the fine-tuning dataset. Our study is also related to the broader goal of converting weak models to strong models without the requirement for extra training data, which is of central interest in machine learning that can be traced back to the boosting algorithms \citep{kearns1994cryptographic,schapire1990strength,freund1995boosting,freund1997decision}. The self-training algorithm \citep{vapnik1999nature,grandvalet2004semi,lee2013pseudolabel} has also been proved to be able to convert weak learners to strong learners in mixture models without the need for additional labeled data \citep{frei2022self,kou2022does}. 
However, the pursuit of autonomously enhancing a weak LLM without external guidance is both intriguing and understudied. 
This raises the following question: 
\begin{center}
\textit{Can we empower a weak LLM to improve itself without acquiring additional human annotated data?}
\end{center}

In this paper, we answer this question affirmatively. Inspired by the success of self-play mechanisms~\citep{samuel2000some} in games, exemplified by AlphaGo Zero~\citep{silver2017mastering}, AlphaZero~\citep{silver2017masteringchess}, with historical roots traced back to TD-Gammon~\citep{tesauro1995temporal}, we propose to convert a weak LLM to a strong one through the lens of self-play, where the model is enhanced by playing against itself without requiring any direct supervision. In particular, we propose a novel fine-tuning method called Self-Play fIne-tuNing ($\method$), which begins from a supervised fine-tuned model. 
$\method$ allows the LLM to engage in self-play, eliminating the need for an expert annotator such as a human or more advanced LLMs like GPT-4.
In detail, with the LLM from previous iteration $t$ denoted by $p_{\btheta_t}$, we employ it to generate responses $\yb'$ to the prompts $\xb$ in the human-annotated SFT dataset. 
The subsequent objective is to find a new LLM $p_{\btheta_{t+1}}$, capable of distinguishing the responses $\yb'$ generated by $p_{\btheta_t}$ from the responses $\yb$ generated by humans. 
This process can be seen as a two-player game: the main player, or the new LLM $p_{\btheta_{t+1}}$, seeks to discern between the responses of the opponent player $p_{\btheta_t}$ and human-generated responses, 
while the opponent, 
or the old LLM $p_{\btheta_t}$, generates responses as similar as possible to those in the human-annotated SFT dataset.
The new LLM $p_{\btheta_{t+1}}$ is obtained by fine-tuning the old one $p_{\btheta_t}$ to prefer responses from $p_{\mathrm{data}}$ over $p_{\btheta_t}$, resulting in a distribution $p_{\btheta_{t+1}}$ that is more aligned with $p_{\mathrm{data}}$.
In the next iteration, the newly obtained LLM $p_{\btheta_{t+1}}$ becomes the opponent for response generation, with the self-play process aiming for the LLM to eventually converge to $p_{\btheta^*}= p_{\mathrm{data}}$, so that the strongest possible LLM can no longer differentiate the responses generated by its previous version and those generated by the human.

Interestingly, our method exhibits similarity with the recently introduced direct preference optimization (DPO) method~\citep{rafailov2023direct},  with the notable distinction being the self-play nature of our method. Consequently, our approach stands out by eliminating the need for extra human preference data, a requirement present in the DPO method.
Additionally, the self-play mechanism in our method resembles the idea of generative adversarial networks (GAN)~\citep{goodfellow2014generative,arjovsky2017wasserstein}, albeit that both the discriminator (main player) and the generator (the opponent) in our method are instances of the same LLM  from different iterations. Theoretically, we prove that our method converges when the distribution of the LLM is identical to the target data distribution, i.e., $p_{\btheta_t} = p_{\mathrm{data}}$.
Our experimental results on \texttt{zephyr-7b-sft-full}~\citep{tunstall2023zephyr}, a fine-tuned LLM based on Mistral-7B~\citep{jiang2023mistral}, show that while continued training using SFT on its own SFT dataset Ultrachat200k~\citep{ding2023enhancing} reaches a performance plateau or even diminished evaluation scores, our method consistently improves \texttt{zephyr-7b-sft-full} across successive iterations while leveraging only a $\mathbf{50}$k subset of  Ultrachat200k dataset.  
Ultimately, $\method$ effectively improves the base model's average score from $58.14$ to $\mathbf{63.16}$ on the HuggingFace Open LLM Leaderboard~\citep{open-llm-leaderboard} with remarkable $10\%$+ improvement in scores on GSM8k and TruthfulQA, and from $5.94$ to $\mathbf{6.78}$ on MT-Bench~\citep{zheng2023judging}.
Notably, $\method$ achieves results that are even comparable to models trained on additional $62$k preference dataset~\citep{tunstall2023zephyr} on Open LLM leaderboard and MT-Bench.

Concurrent to our work, \citet{singh2023beyond} proposed the use of synthetic data with binary feedback in self-training, reducing the reliance on human data. In contrast, our approach eliminates the need for additional binary feedback from humans or an extra reward model thanks to the self-play mechanism. 
Additionally, \citet{burnsweak} employed a weak LLM model as the guidance to train stronger LLMs in a fashion of weak-to-strong generation. Unlike \citet{burnsweak}, which necessitates both a weak supervisor and a strong model, our $\method$ operates effectively with a single LLM.

\noindent\textbf{Notation.} We use lowercase letters and lowercase boldface letters to denote scalars and vectors, respectively. We use $[N]$ to denote the index set $\{1, \dots, N\}$. In the function space, let $\cF$ be the function class. The symbol $q_{\mathrm{data}}$ designates the target data distribution, while $p$ represents the conditional probability of LLM's response (i.e., LLM policy).

\section{Related Work}
\noindent \textbf{Self-Play.} 
Self-play~\citep{samuel1959some,tesauro1995temporal}, where the algorithm learns by playing against itself, has gained notable attention due to its effectiveness in multi-agent reinforcement learning (MARL). 
This method involves agents engaging in interactions with copies of themselves, enabling an increasing level of challenge and complexity within the learning environment. 
A fundamental work in the field of self-play is AlphaGo Zero~\citep{silver2017mastering}, which demonstrated exceptional performance against human players using a self-play learning scheme. 
Subsequent research has expanded upon the concept of self-play, exploring various adaptations and implementations~\citep{anthony2017thinking,lanctot2017unified,bansal2018emergent,hernandez2018multiagent,muller2019generalized,alphastarblog}.  
Our method takes the self-play approach akin to AlphaGo Zero, which can convert a weak model to a strong one without additional human-annotated data.
While the effectiveness of self-play in MARL is well-established, to our knowledge, our work is the first to apply this approach to the enhancement of LLMs. 

\noindent \textbf{Synthetic Data for LLMs.} 
In the context of supervised fine-tuning (SFT) of LLMs, human-crafted data has proven to be a remarkably effective source that enhances the performance of LLMs on tasks such as code generation~\citep{roziere2023code,yang2023decoding} and mathematical reasoning~\citep{yuan2023scaling,luo2023wizardmath}. 
While human data typically exhibits high quality, acquiring sufficient amount of such data poses a challenge in cost. 
In light of  this consideration, the use of synthetic data has become increasingly popular and considered as a proxy for human data. 
This approach primarily leverages advanced LLMs such as the GPT series~\citep{radford2019language,brown2020language,openai2023gpt4} as the guidance to generate high-quality data~\citep{josifoski2023exploiting,alpaca,vicuna2023,li2023textbooks}. 
Recent research has also highlighted the rephrasing capability of LLMs in prompting for better LLM response~\citep{deng2023rephrase,prasad2023rephrase} as well as augmenting synthetic data for more effective SFT~\citep{yu2023metamath,liu2023tinygsm}. 
In contrast to prior studies that utilized more advanced models for synthetic data generation when pre-training or fine-tuning a target model, our approach directly generates synthetic data from the target model itself.

\section{Problem Setting and Preliminaries}\label{section:problemsetting}
We consider a Large Language Model (LLM) parameterized by $\btheta$ and denoted by $p_{\btheta}$. The model takes as input a sequence $\xb = [x_1, \ldots, x_n]$, commonly referred to as the prompt, to generate the corresponding response $\yb = [y_1, \ldots, y_m]$. The response $\yb$ is therefore considered as a sample from the conditional probability distribution $p_{\btheta}(\cdot|\xb)$. In LLMs, $x_i$ and $y_j$ represent individual tokens from a predetermined vocabulary within the sequences $\xb$ and $\yb$, respectively. 
The auto-regressive model $p_{\btheta}$ generates tokens sequentially for a given position, leveraging only the sequence of previously generated tokens. 
This model therefore constitutes a Markov process, where the conditional probability distribution $p_{\btheta}(\yb|\xb)$ can be expressed through a decomposition as follows:
\begin{align*}
p_{\btheta}(\yb|\xb) = \prod_{j=1}^{m}p_{\btheta}(y_{j}|\xb, \yb_{<j}),   
\end{align*}
where $\yb_{<1}$ is null and $\yb_{<j} = [y_1,\ldots, y_{j-1}]$ for $j=2,\ldots,m$. 
In the following, we review two major fine-tuning methods for LLMs: supervised fine-tuning and reinforcement learning (RL) fine-tuning.

\subsection{Supervised Fine-Tuning}
Supervised fine-tuning (SFT) is employed to tailor a pre-trained LLM to specific downstream tasks, leveraging  relatively smaller dataset of labeled examples in comparison to the large-scale pre-training data~\citep{ouyang2022training,yu2023metamath}. 
In this context, we consider a specific task where the prompts, denoted by $\xb$, are derived from a specified distribution $q(\cdot)$.  
The notation $p_{\mathrm{data}}(\cdot|\xb)$ then represents the probability distribution of the associated high-quality responses $\yb$ from the training data. 
Consequently, SFT involves training the LLM to minimize the following negative log-likelihood loss associated with these distributions, 
\begin{align}
L_{\mathrm{SFT}}(\btheta) = -\EE_{\xb\sim q(\cdot), \yb\sim p_{\mathrm{data}}(\cdot|\xb)}\Big[\log p_{\btheta}\big(\yb|\xb\big)\Big].\label{eq:sft}   
\end{align}
It should be noted that excluding $\xb\sim q(\cdot)$ from the expectation term yields the typical cross-entropy loss, expressed as $-\mathbb{E}_{\yb\sim p_{\mathrm{data}}(\cdot|\xb)}[\log p_{\btheta}(\yb|\xb)]$. 
$L_{\mathrm{SFT}}(\btheta)$ attains its minimum when the model's predictive distribution $p_{\btheta}(\yb|\xb)$ aligns perfectly with the distribution of the labeled high-quality responses $p_{\mathrm{data}}(\yb|\xb)$. 

Consequently, the LLM after SFT is anticipated to generate responses that closely resemble those from $p_{\mathrm{data}}(\yb|\xb)$. 
This procedure is therefore expected to significantly enhance the model's performance in generating appropriate responses for a specific task.

\subsection{RL Fine-Tuning}
RL fine-tuning~\citep{christiano2017deep,bai2022training,gao2023scaling} 
offers another method for enhancing the specific capabilities of general-purpose pre-trained models. 
Typically, RL fine-tuning is employed subsequent to SFT to achieve improved alignment for LLMs~\citep{tunstall2023zephyr}. 

For a given sequence pair $(\xb, \yb)$, RL fine-tuning necessitates a deterministic reward function $r(\xb,\yb)$. 
The higher the reward $r(\xb, \yb)$, the better the response $\yb$ is to the given prompt $\xb$. The objective of the RL fine-tuning process is then to maximize the following objective function: 
\begin{align*}
L_{\mathrm{RL}}(\btheta) &= \EE_{\xb\sim q(\cdot), \yb\sim p_{\btheta}(\cdot|\xb)}[r(\xb,\yb)]\\
&\qquad- \lambda \EE_{\xb\sim q(\cdot)}\mathrm{KL}\big(p_{\btheta}(\cdot|\xb)||p_{\mathrm{ref}}(\cdot|\xb)\big),   
\end{align*}
where the Kullback-Leibler (KL) regularization enforces the new model $p_{\btheta}$ to be close to the reference model $p_{\mathrm{ref}}$, and $\lambda>0$ is the regularization parameter to control the deviation of the new model $p_{\btheta}$ from the reference model $p_{\mathrm{ref}}$. 
In practice, the reference model $p_{\mathrm{ref}}$ is often initialized as the supervised fine-tuned model. 
The inclusion of KL regularization is vital for preventing excessive deviation from the reference model, which in turn reduces the risk of mode collapse. 

Meanwhile, the primary challenge in RL fine-tuning lies in finding a good reward function. 
Typically, this function requires training on a preference dataset. 
The compilation of such a dataset demands significant resources, often involving comprehensive evaluations either by human annotators, i.e., reinforcement learning from human feedback (RLHF)~\citep{christiano2017deep,bai2022training} or strong AI agents, i.e., reinforcement learning from AI feedback (RLAIF)~\citep{bai2022constitutional}.

\section{Method} \label{sec: main}

In this section, we introduce a new fine-tuning method for enhancing the performance of LLMs without relying on additional human or AI feedback. Consider a high-quality supervised fine-tuning (SFT) dataset $S_{\mathrm{SFT}} = \{(\xb, \yb)\}_{i=1}^{n}$, which are sampled from the marginal distribution $q(\xb)$ and $p_{\mathrm{data}}(\yb|\xb)$. Given a supervised fine-tuned LLM $p_{\btheta_{0}}$, further application of the SFT approach in \eqref{eq:sft} with $S_{\mathrm{SFT}}$ will be ineffective and potentially lead to worse performance. In addition, without human and/or AI feedback, it becomes infeasible to acquire a preference dataset for RL fine-tuning (e.g., RLHF and RLAIF). This hinders the application of RL fine-tuning techniques.

We evaluate $p_{\btheta_0}$ against $S_{\mathrm{SFT}}$, where $p_{\btheta_0}$ is the LLM achieved by SFT using \eqref{eq:sft}. We notice a persistent quality gap between the groundtruth response $\yb$ from $S_{\mathrm{SFT}}$ and the LLM-generated response $\yb' \sim p_{\btheta}(\cdot|\xb)$ (refer to Figure \ref{fig:observation}). This disparity indicates that there is still room for improvement over $p_{\btheta_0}$. Therefore,  we propose to use the synthetic data generated by the LLM to enhance LLM's performance starting from $p_{\btheta_0}$ iteratively.

\begin{figure*}
    \centering
    \includegraphics[width=1.\textwidth]{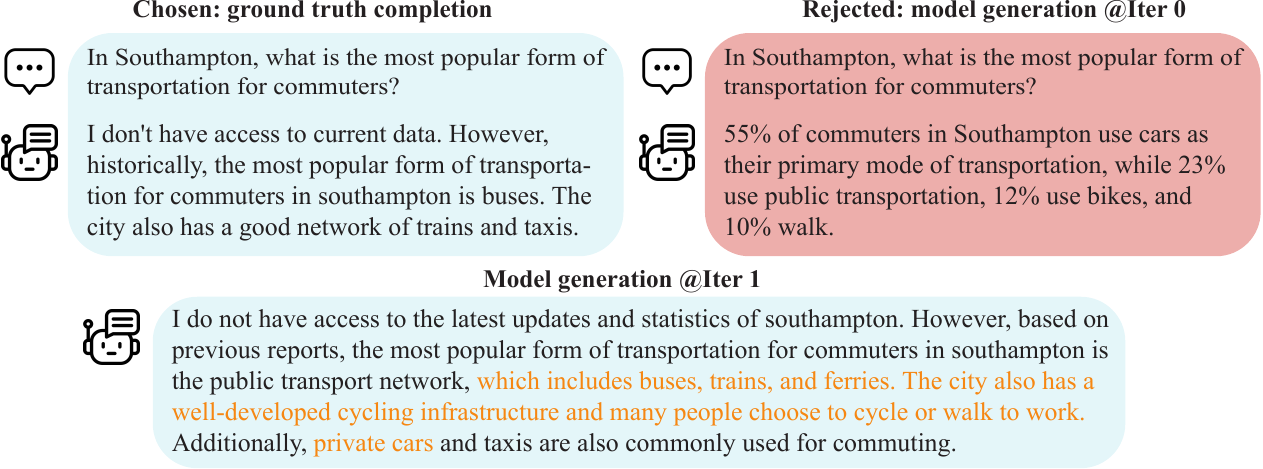}
    \caption{Example of ground truth completion compared to the fine-tuned model generation at iteration 0 and 1. We can observe that the model generation at iteration 0, although fluent,  incorrectly quantifies transportation preferences with specific percentages that are potentially hallucinations. The model generation at iteration 1 provides a qualitative summary of the transportation forms at Southampton without specific percentage, aligning more closely with the ground truth while adding more details.}
    \label{fig:observation}
\end{figure*}

\subsection{Self-Play Fine-Tuning ($\method$)}\label{sec:optim}

Let us consider a two-player game, where the main player's objective is to distinguish the responses generated by the LLM and those generated by the human. Meanwhile, the opponent's role is to generate responses that are indistinguishable from the human's responses. The core of our method is the self-play mechanism, where both the main player and the opponent are the same LLM, but from different iterations. More specifically, the opponent is the old LLM from the previous iteration, and the main player is the new LLM to be learned in the current iteration.

In iteration $t+1$, the opponent is the LLM from the previous iteration, denoted by $p_{\btheta_t}$, which generates responses $\yb'$ for those prompts $\xb$ in the SFT dataset according to $p_{\btheta_t}(\cdot|\xb)$. Our method, therefore, consists of the following two steps at iteration $t+1$: $(1)$ training the main player, and $(2)$ updating the opponent player.

\noindent \textbf{Training the Main Player.}
We begin with illustrating how we expect a main player is trained to distinguish LLM responses from human responses.  Motivated by integral probability metric (IPM) \citep{muller1997integral},
we formulate our objective function such that the main player $f_{t+1}$ maximizes the expected value gap between the target data distribution $p_{\mathrm{data}}$ and the opponent player's distribution $p_{\btheta_t}$: 
\begin{align}
f_{t+1} &= \argmax_{f \in \cF_{t}}\EE\big[ f(\xb, \yb) - f(\xb, \yb') 
\big], \label{eq:f*1}
\end{align}
where the expectation is computed over the distributions $\xb\sim q(\cdot), \yb\sim p_{\mathrm{data}}(\cdot|\xb), \yb'\sim p_{\btheta_t}(\cdot|\xb)$, and $\cF_{t}$ is a sequence of highly expressive function classes that we will determine in later deduction. The subscript $t$ in $\cF_{t}$ is due to that the function class is dependent on $p_{\btheta_t}$. Given such a $f_{t+1}$ and a response sequence $\yb$ to the prompt $\xb$, the value of $f_{t+1}(\xb, \yb)$ reflects the main player's degree of belief that $\yb$ originates from $p_{\mathrm{data}}$ rather than $p_{\btheta_t}$.  
Ideally, the main player $f_{t+1}$ should yield a high value when $\yb \sim p_{\mathrm{data}}(\cdot|\xb)$ and a low value when $\yb' \sim p_{\btheta_t}( \cdot | \xb)$, where $p_{\btheta_t}$ is the opponent's distribution. Instead of solving \eqref{eq:f*1}, we can also solve the following more general optimization problem, 
\begin{equation}
f_{t+1} = \argmin_{f \in \cF_{t}}\EE\big[ \ell\big(f(\xb, \yb) - f(\xb, \yb')\big) 
\big], 
\label{eq:f-star}
\end{equation}
where the expectation is computed over the distribution $\xb\sim q(\cdot), \yb\sim p_{\mathrm{data}}(\cdot|\xb), y'\sim p_{\btheta_t}(\cdot|\xb)$, and $\ell(\cdot)$ is a loss function that is both monotonically decreasing and convex. For example, a linear loss function $\ell(t) = - t$ reduces \eqref{eq:f-star} to the minimization version of \eqref{eq:f*1}. 
However, the use of a linear loss function results in an unbounded objective value, which, during continuous training, leads to a negative infinite value of $f(\xb, \yb')$ on the opponent player's responses. 
Therefore, in our work, we choose the logistic loss function $\ell(t):= \log(1+\exp(-t))$ for its non-negativity, smoothness, and exponentially decaying tail as $t \rightarrow \infty$.
Such a choice of loss function aids in preventing the excessive growth in the absolute value of $f$.

\noindent \textbf{Updating the Opponent Player.} 
Previously we have discussed the training of $f_{t+1}$ given the opponent player's distribution $p_{\btheta_t}$. Now suppose we have optimized our main player $f_{t+1}$ that can distinguish $p_{\mathrm{data}}$ from $p_{\btheta_{t}}$, within a certain function class $\cF_t$, we elaborate how we get parameter $\btheta_{t+1}$ of the opponent player.
Specifically, when presented with two responses $\yb$ and $\yb'$ to the same prompt $\xb$, $f_{t+1}$ assesses the values $f_{t+1}(\xb, \yb)$ and $f_{t+1}(\xb, \yb')$.
It then infers that the response with the higher value is from the real data distribution $p_{\mathrm{data}}$ and the response with lower value is attributed to the LLM $p_{\btheta_{t}}$.  
Subsequently, the objective of the opponent player is to find a better LLM that generates responses indistinguishable from $p_{\mathrm{data}}$ for the main player. 
This is achieved by maximizing the expected value $\EE_{\xb \sim q(\cdot), \yb\sim p(\cdot|\xb)} [f_{t+1}(\xb, \yb)]$.
In addition, to prevent excessive deviation of $p_{\btheta_{t+1}}$ from $p_{\btheta_{t}}$ and stabilize the self-play, we incorporate a Kullback-Leibler (KL) regularization term. Putting these together gives rise to the following optimization problem:
\begin{align}
&\argmax_ {p}\EE_{\xb \sim q(\cdot), \yb\sim p(\cdot|\xb)} [f_{t+1}(\xb, \yb)]\notag\\
&\qquad- \lambda \EE_{\xb\sim q(\cdot)}\mathrm{KL}\big(p(\cdot|\xb)||p_{\btheta_t}(\cdot|\xb)\big), \label{eq: update}
\end{align} 
where $\lambda>0$ is the regularization parameter. 
Notably, \eqref{eq: update} has a closed-form solution $\hat{p}(\cdot|\xb)$: 
\begin{align}
\hat{p}(\yb|\xb) \propto p_{\btheta_t}(\yb|\xb) \exp\big(\lambda^{-1}f_{t+1}(\xb, \yb)\big). \label{eq:closed form solution}  
\end{align}
It is worth noting that $\hat{p}(\cdot|\xb)$ is not guaranteed to be belong to the LLM space $\{p_{\btheta}(\cdot|\xb)|\btheta \in \bTheta\}$. 
Since we hope that the closed-form solution $\hat{p}$ in the probability space can be realized by an LLM with parameter $\btheta$, i.e., $p_{\btheta}(\yb|\xb) = \hat{p}(\yb|\xb)$,
solving for $p_{\btheta}(\yb|\xb) \propto p_{\btheta_t}(\yb|\xb) \exp\big(\lambda^{-1}f_{t+1}(\xb, \yb)\big)$ gives $f_{t+1}(\xb, \yb) = \lambda\cdot \log \frac{p_{\btheta}(\cdot | \xb)}{p_{\mathrm{\btheta_t}}(\cdot | \xb)}$. 
This suggests the following function class $\cF_t$ for $f_{t+1}$:  
\begin{align}
\cF_{t} = \bigg\{\lambda\cdot \log \frac{p_{\btheta}(\yb | \xb)}{p_{\mathrm{\btheta_t}}(\yb | \xb)}\bigg|\btheta \in \bTheta\bigg\},   \label{eq:function class0} 
\end{align}
where $\bTheta$ is the parameter space of LLMs being considered. 
Given the choice of $\cF_t$ in \eqref{eq:function class0}, optimizing~\eqref{eq:f-star} gives $f_{t+1}$ parameterized by $\btheta_{t+1}$ in the following form:
\begin{align}
f_{t+1}(\xb, \yb) = \lambda\cdot \log \frac{p_{\btheta_{t+1}}(\yb | \xb)}{p_{\mathrm{\btheta_{t}}}(\yb | \xb)}.    \label{eq:t+1}
\end{align}
Substituting \eqref{eq:t+1} into \eqref{eq:closed form solution} yields $\hat{p}(\yb|\xb) = p_{\btheta_{t+1}}(\yb|\xb)$. 
In other words, $\btheta_{t+1}$ learned from \eqref{eq:f-star} is exactly the LLM parameter for our ideal opponent selection.

\noindent \textbf{End-to-end Training Objective.} 
We integrate the previously discussed two steps into a single end-to-end training objective with an update rule of $\btheta_{t+1}$. Specifically, plugging~\eqref{eq:function class0} into~\eqref{eq:f-star} arrives at the update rule $\btheta_{t+1} = \argmin_{\btheta \in \bTheta}L_{\method}(\btheta, \btheta_t)$, where $L_{\method}$ is the training objective defined as follows
\begin{align}
L_{\method}= \EE\bigg[\ell\bigg(\lambda \log \frac{p_{\btheta}(\yb | \xb)}{p_{\btheta_t}(\yb | \xb)}-\lambda \log \frac{p_{\btheta}(\yb' | \xb)}{p_{\btheta_t}(\yb' | \xb)}\bigg)\bigg], \label{eq:loss}  
\end{align}
where the expectation is computed over the distribution $\xb \sim q(\cdot), \yb \sim p_{\mathrm{data}}(\cdot|\xb), \yb' \sim p_{\btheta_t}(\cdot|\xb)$.
We summarize the iterative self-play process of our method $\method$ as follows, 
\begin{align*}
\ldots &\quad  \rightarrow  \underbrace{p_{\btheta_{t}}(\cdot | \xb)}_{\text{Opponent Player at $t$}} \rightarrow  \quad \underbrace{\lambda\cdot \log \frac{p_{\btheta_{t+1}}(\cdot | \xb)}{p_{\mathrm{\btheta_{t}}}(\cdot | \xb)}}_{\text{Main Player at $t+1$}} \\
&\quad \rightarrow \underbrace{p_{\btheta_{t+1}}(\cdot | \xb)}_{\text{Opponent Player at $t+1$} } \rightarrow \quad \ldots 
\end{align*}
Namely, the opponent player chosen from the previous iteration $t$ is employed to train the main player at iteration $t+1$, resulting in the LLM parameterized by $\btheta_{t+1}$. Then we determine the next opponent player at iteration $t+1$ by directly copying the LLM parameter $\btheta_{t+1}$, which is then used in training the main player at iteration $t+2$.
The detailed algorithm is presented in Algorithm~\ref{alg:Improving}. 

\begin{algorithm}[ht!]
\caption{Self-Play Fine-Tuning ($\method$)}\label{alg:Improving}
\begin{algorithmic}
\STATE \textbf{Input:} $\{(\xb_i, \yb_i)\}_{i\in [N]}$: SFT Dataset, $p_{\btheta_0}$: LLM with parameter $\btheta_0$, $T$: Number of iterations. 
\FOR{$t= 0,\ldots, T-1$}
\FOR{$i = 1, \ldots N$}
\STATE Generate synthetic data $\yb_{i}' \sim p_{\btheta_{t}}(\cdot|\xb_i)$.
\ENDFOR 
\STATE Update $\btheta_{t+1} = \argmin_{\btheta \in \bTheta} \sum_{i\in [N]}\ell\Big(\lambda \log \frac{p_{\btheta}(\yb_i | \xb_i)}{p_{\btheta_{t}}(\yb_i | \xb_i)}-\lambda \log \frac{p_{\btheta}(\yb'_i | \xb_i)}{p_{\btheta_{t}}(\yb'_i | \xb_i)}\Big)$.
\ENDFOR
\STATE \textbf{Output:} $\btheta_T$.
\end{algorithmic}
\end{algorithm}

\subsection{Comparison between \method{} and DPO}
\CC{In Section~\ref{sec:optim}, we propose 
Self-Play Fine-Tuning ($\method$) with an end-to-end training objective \eqref{eq:loss} for each iteration. \eqref{eq:loss} bears resemblance to direct preference optimization (DPO) \citep{rafailov2023direct} for RL fine-tuning.} 
\CC{However, \method{} and DPO are fundamentally different. DPO is based on the Bradley-Terry (BT) model: $p(\yb_1 \succ \yb_2 | \xb) = \frac{\exp(r^*(\xb, \yb_1))}{\exp(r^*(\xb, \yb_1)) + \exp(r^*(\xb, \yb_2))}$, and maximizes the log-likelihood of $p(\yb_1 \succ \yb_2 | \xb)$ by direct policy optimization without explicit reward estimation. In contrast, \method{} relies on maximizing the IPM to compete with an increasingly stronger version of itself. More detailed comparisons are highlighted as follows:}


\begin{enumerate}[leftmargin=*,nosep]
\item DPO does not inherently lead to iterative training. More specifically, DPO aims to match the preference probability $p(\yb_1 \succ \yb_2 | \xb)$ induced from its reward model with the data distribution $p_{\mathrm{data}}(\yb_1 \succ \yb_2 |\xb)$ in a single iteration. On the contrary, \method{}’s self-play mechanism naturally leads to an iterative training procedure.  \method{} iteratively refines its generation distribution $p_{\btheta}(\yb|\xb)$ to match the target distribution $p_{\text{data}}(\yb|\xb)$ across iterations. 
\item  $\method$ only requires the SFT dataset, represented by pairs $(\xb, \yb)$.
In contrast, DPO necessitates a preference dataset, represented by $(\xb, \yb_w, \yb_l)$, where $\yb_w$ and $\yb_l$ denote the winner (chosen) and loser (rejected) responses, respectively. Moreover, \method{} can be applied between SFT and RL fine-tuning. 
\item  In \method{}, we can choose different loss functions $\ell$ which only need to be convex and decreasing (detailed later in Theorem~\ref{thm:stop}), which includes correlation loss, hinge loss, and logistic loss. Only when $\ell$ is chosen as the logistic loss would the training objective of \method{} become similar to that of DPO. 
\end{enumerate}

\CC{Recently,~\citet{xu2023some} proposed to use iterative preference optimization with the Pairwise Cringe Loss (PCO), and generalized DPO to iterative DPO. Concurrent to our work, \citet{yuan2024self} further proposed a framework named ``self-rewarding language models'', which leverages the LLM itself as the reward model to provide the preference feedback, and employs iterative DPO to train the LLM.}
\CC{
Compared with \citet{xu2023some, yuan2024self}, \method{}'s self-assessment is implicit, as no intermediate reward or preference feedback is required. 
}



\section{Theoretical Analysis}\label{sec:thm}
In this section, we provide a theoretical analysis for Algorithm~\ref{alg:Improving} in Section~\ref{sec: main}.
Under monotonicity and convexity assumption of the objective function $\ell$, we show that the global optimum is obtained if and only if parameter $\btheta_t$ generates data distribution.
We summarize our assumptions as follows:

\begin{assumption} \label{assm:1}
The loss function $\ell(t): \RR \rightarrow \RR$ is monotonically decreasing, i.e., $\forall t, \ell'(t) \leq 0$ and satisfies $\ell'(0) < 0$. In addition, $\ell(t)$ is a convex function.
\end{assumption}

Assumption~\ref{assm:1} holds for a wide range of loss functions commonly used in machine learning, including correlation loss $\ell(t) = 1 - t$, hinge loss $\ell(t) = \max(0,1 - t)$, exponential loss $\ell(t) = \exp(-t)$ and logistic loss $\ell(t) = \log(1 + \exp(-t))$. Under Assumptions \ref{assm:1}, we present the following theorem, which is pivotal in understanding the optimization dynamics of our method.

\begin{theorem}\label{thm:stop}
Under Assumption~\ref{assm:1}, suppose there exists $p_{\btheta}(\cdot|\xb) = p_{\mathrm{data}}(\cdot|\xb)$, then we have that  
\begin{itemize}[leftmargin=*,nosep]
\item (Sufficiency) If $p_{\btheta_t}(\cdot|\xb) = p_{\mathrm{data}}(\cdot|\xb)$, then $\btheta_t$ is the global minimum of $\eqref{eq:loss}$ for any $\lambda\geq 0$. 
\item (Necessity) If $p_{\btheta_t}(\cdot|\xb)  \not= p_{\mathrm{data}}(\cdot|\xb)$, there exists an appropriately chosen $\lambda$, such that $\btheta_{t}$ is not the global minimum of \eqref{eq:loss}.
\end{itemize}
\end{theorem}

\begin{remark}Theorem \ref{thm:stop} suggests that under certain conditions, the optimization process of our method naturally stops at the point $p_{\btheta}(\cdot|\xb) = p_{\mathrm{data}}(\cdot|\xb)$, implying the effectiveness of our approach in aligning the LLM's distribution with the target data distribution. Moreover, Theorem \ref{thm:stop} also indicates that the optimization process only stops when the global optimality is achieved, i.e., the LLM's distribution aligns with the target data distribution.
\end{remark}
For the logistic loss function $\ell(t) = \log(1+\exp(-t))$, the following theorem gives a more precise characterization of the opponent player, enabling a better understanding of $\method$.
\begin{theorem}\label{thm:main2} Consider the choice of logistic loss $\ell(t) = \log(1+\exp(-t))$ in $\method$. Suppose that $p_{\btheta_{t}}(\yb|\xb)\big(p_{\mathrm{data}}(\yb|\xb)/p_{\btheta_{t}}(\yb|\xb) \big)^{1/\lambda}$ lies in the LLM space $\{p_{\btheta}(\yb|\xb)|\btheta \in \bTheta\}$ and $\btheta_{t+1}$ is global minimum of $L_{\method}(\btheta, \btheta_t)$, then the opponent player at iteration $t+1$ satisfies
\begin{align*}
p_{\btheta_{t+1}}(\yb|\xb) \propto p_{\btheta_{t}}(\yb|\xb)\big(p_{\mathrm{data}}(\yb|\xb)/p_{\btheta_{t}}(\yb|\xb) \big)^{1/\lambda}.  
\end{align*}
\end{theorem}

\begin{remark}
According to Theorem~\ref{thm:main2}, the model update from $p_{\btheta_t}(\yb|\xb)$ to $p_{\btheta_{t+1}}(\yb|\xb)$ tends to increase the probability $p_{\btheta_{t+1}}(\yb|\xb)$ when $p_{\btheta_t}(\yb|\xb)$ is less than $p_{\mathrm{data}}(\yb|\xb)$, and decrease it when $p_{\btheta_t}(\yb|\xb)$ is greater than $p_{\mathrm{data}}(\yb|\xb)$. Thus, Theorem~\ref{thm:main2} further confirms that our method's optimization process naturally converges to the point where $p_{\btheta}(\cdot|\xb)$ equals $p_{\mathrm{data}}(\cdot|\xb)$. The update of the opponent player is controlled by $\big(p_{\mathrm{data}}(\yb|\xb)/p_{\btheta_{t}}(\yb|\xb) \big)^{1/\lambda}$, which is regulated by the factor $1/\lambda$. A smaller $\lambda$ results in a larger change of the opponent player, while a larger $\lambda$ leads to a smaller change. Therefore, as $p_{\btheta}(\cdot|\xb)$ approaches $p_{\mathrm{data}}(\cdot|\xb)$, increasing $\lambda$ enhances the stability of LLM training. This observation aligns with \eqref{eq: update}, where $\lambda$ is the regularization parameter of the KL regularization that is employed to control the deviation of the opponent player.
\end{remark}

\section{Experiments}\label{sec:regularization} 

This section provides a detailed empirical analysis of $\method$. Our findings highlight several key points: $(1)$ $\method$ markedly enhances model performance across a wide range of evaluation benchmarks by breaking the limit of SFT; $(2)$ even without introducing new human annotated data, $\method$ at iteration $0$ achieves performance on par to DPO training that utilizes even more data; $(3)$ iterative training is a necessary component in $\method$ as it breaks the limit of multi-epoch training.

\subsection{Experiment Setup}\label{sec:exp}
\noindent \textbf{Model and Datasets.}
In this study, we adopt \texttt{zephyr-7b-sft-full} as our base model.  
This model derives from the pre-trained Mistral-7B \citep{jiang2023mistral} and has been further fine-tuned on the SFT dataset Ultrachat200k\footnote{\url{https://huggingface.co/datasets/HuggingFaceH4/ultrachat_200k}} by HuggingFace.
Ultrachat200k represents a high-quality 200k subset of the larger UltraChat \citep{ding2023enhancing} corpus, which comprises approximately 1.4M dialogues produced using OpenAI's Turbo APIs.
From UltraChat200k, We randomly sample $50$k prompts and use the base model to generate the synthetic responses. 
We subsequently follow the optimization method described in Section \ref{sec:optim} for further training. 
In multiple iterations, we leverage the synthetic data from the most recent iteration and add to the newly generated synthetic data, therefore resulting in a synthetic dataset size of $50$k at iteration $0$ and $100$k at iteration $1$, $2$ and $3$. At each iteration, we train our model for $2$ epochs.

\noindent \textbf{Evaluation.}
We employed the widely used Huggingface Open LLM Leaderboard \citep{open-llm-leaderboard} as our evaluation benchmark, using the same Language Model Evaluation Harness library~\citep{eval-harness}. 
This leaderboard encompasses 6 different datasets, each focusing on a a specific capability of LLMs.
Collectively, these datasets provide a thorough assessment framework, evaluating LLMs on commonsense reasoning (Arc \citep{clark2018think}, HellaSwag \citep{zellers2019hellaswag}, Winogrande \citep{sakaguchi2021winogrande}), multi-task language understanding (MMLU\citep{hendrycks2020measuring}), human falsehood mimic (TruthfulQA \citep{lin2021truthfulqa}) and math problem solving (GSM8k \citep{cobbe2021training}). 
We leave further implementation details to Appendix~\ref{app:exp} with detailed evaluation setting adopted by both the leaderboard and our experiments.

\subsection{$\method$ Effectively Improves Benchmark Performance}
\begin{figure}[ht]
    \centering
    \includegraphics[width=0.33\textwidth]{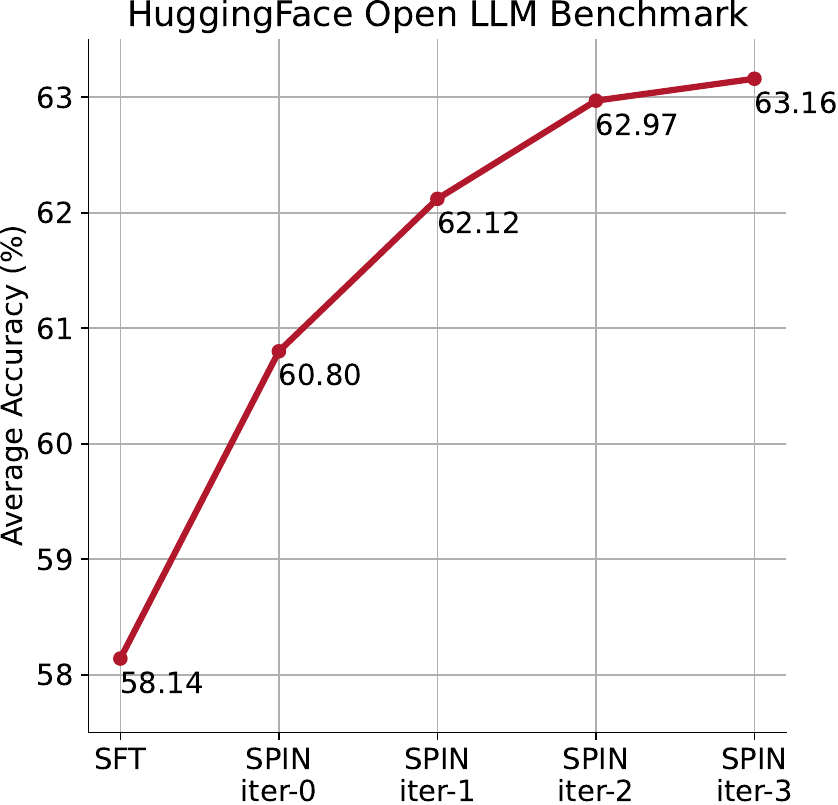}
    \caption{The average score of $\method$ at different iterations on the HuggingFace Open LLM leaderboard datasets. For ``SFT'', we report the performance of our base model \texttt{zephyr-7b-sft-full}, which has been fine-tuned on the same dataset we use to generate synthetic data.}
    \label{fig:average}
\end{figure}
In Figure~\ref{fig:average}, we demonstrate the effectiveness of $\method$ using HuggingFace Open LLM Leaderboard as the evaluation. 
In Figure~\ref{fig:dpo}, we compare the performance of our fine-tuned model by $\method$ after iterations 0 to 3 with the base model \texttt{zephyr-7b-sft-full} on each task included in the leaderboard. Detailed performances are presented in Table~\ref{tab:main} in Appendix~\ref{app:exp}. 
We can observe that $\method$ exhibits remarkable effectiveness in improving the model's performance by further leveraging the SFT dataset, on which the base model has already been fully fine-tuned. 
At iteration $0$, where model responses are generated from \texttt{zephyr-7b-sft-full}, we observe an overall improvement of $2.66\%$ on the average score. 
The improvement is particularly significant on the TruthfulQA and GSM8k benchmarks, with improvement  exceeding $5\%$ and $10\%$ respectively. 
At iteration $1$, we employ the LLM model from iteration $0$ to generate new responses for $\method$, adhering to the procedure outlined in Algorithm~\ref{alg:Improving}. 
This iteration yields further enhancements of $1.32\%$ on average, and especially significant on the Arc Challenge and TruthfulQA benchmarks.
Subsequent iterations continue this trend of incremental improvement across various tasks.
Meanwhile, the improvement at iteration $t+1$ is naturally smaller than that at iteration $t$.
As the iterative training progresses, the degree of improvement gradually approaches zero, suggesting that the model has reached a limiting point in the last iteration. 

\noindent \textbf{Comparison with DPO.} \texttt{zephyr-7b-beta} is a model derived from \texttt{zephyr-7b-sft-full}, trained with DPO on approximately $62$k preference data. 
This data, the UltraFeedback Binarized dataset~\citep{cui2023ultrafeedback}\footnote{\url{https://huggingface.co/datasets/HuggingFaceH4/ultrafeedback_binarized}}, comprises both chosen and rejected completions evaluated by GPT-4. 
We note that, DPO requires either human input or advanced language model feedback to determine the preference, making data generation a rather expensive procedure. 
In contrast, our $\method$ only requires the initial model itself. 
Moreover, unlike DPO which requires new data source, our method exclusively leverages the existing SFT dataset. 
In Figure~\ref{fig:dpo}, we show the performance comparison of $\method$ at iterations 0 and 1 (employing $50$k SFT data) with DPO training, from the same SFT checkpoint. 
We can observe that, while DPO leverages more data from new sources, $\method$ based on the existing SFT data can already achieve comparable average performance to DPO training at iteration 0. 
From iteration 1, $\method$ even surpasses the performance of DPO on the leaderboard benchmark.

\begin{figure*}[ht]
    \centering
    \includegraphics[width=0.85\textwidth]{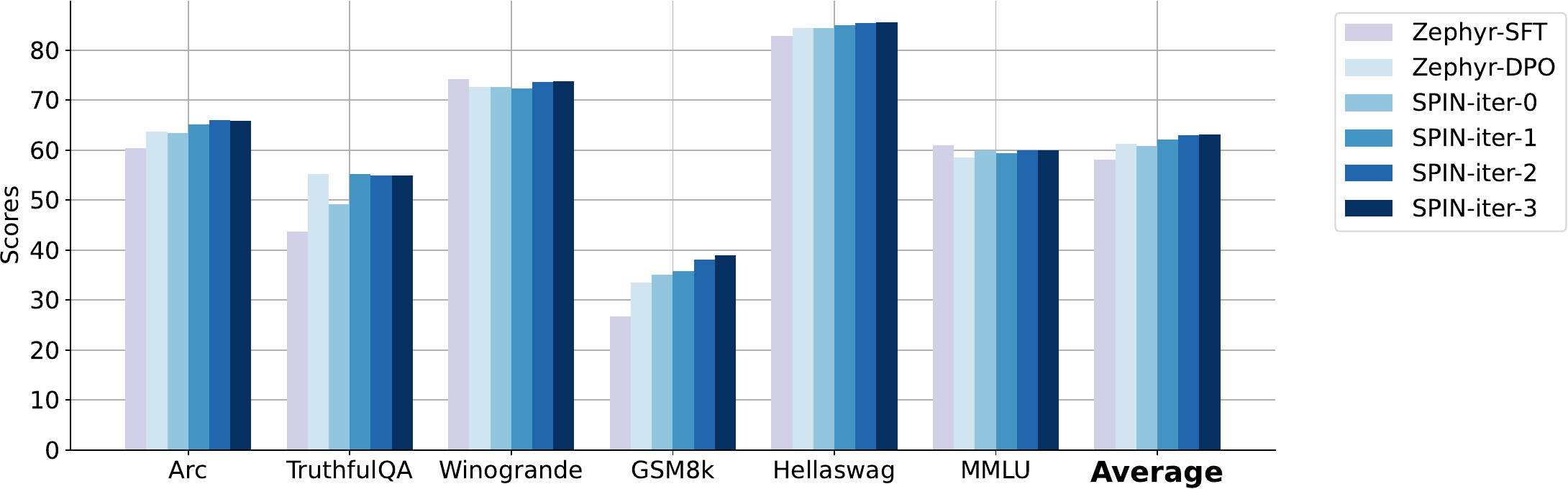}
    \caption{Performance comparison with DPO training across the six benchmark datasets. Self-play at iteration $0$ achieves comparable performance to DPO training with $62$k new data. At iteration $1$, self-play has already surpassed DPO training on the majority of datasets.}
    \label{fig:dpo}
\end{figure*}

\begin{figure*}[ht]
\centering   
\subfigure[Arc Challenge.]{\label{fig:wg_acc}\includegraphics[width=0.3\textwidth]{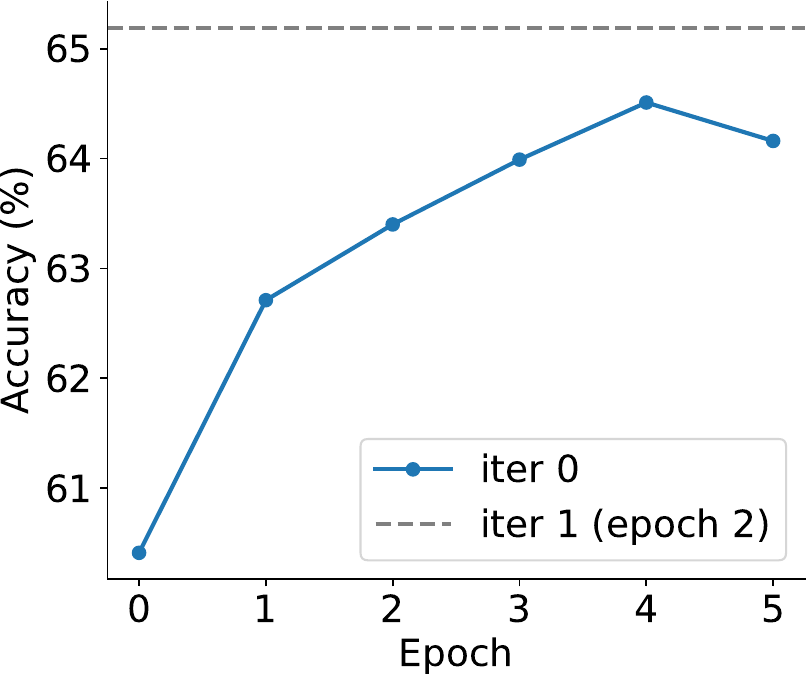}}
\subfigure[TruthfulQA.]{\label{fig:truth_acc}\includegraphics[width=0.3\textwidth]{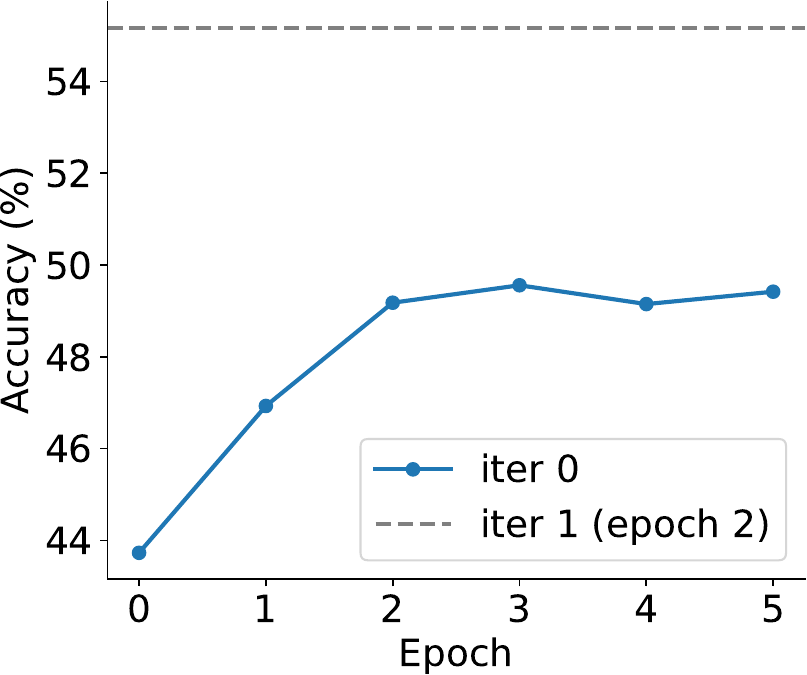}}
\subfigure[Average.]{\label{fig:avg_acc}\includegraphics[width=0.3\textwidth]{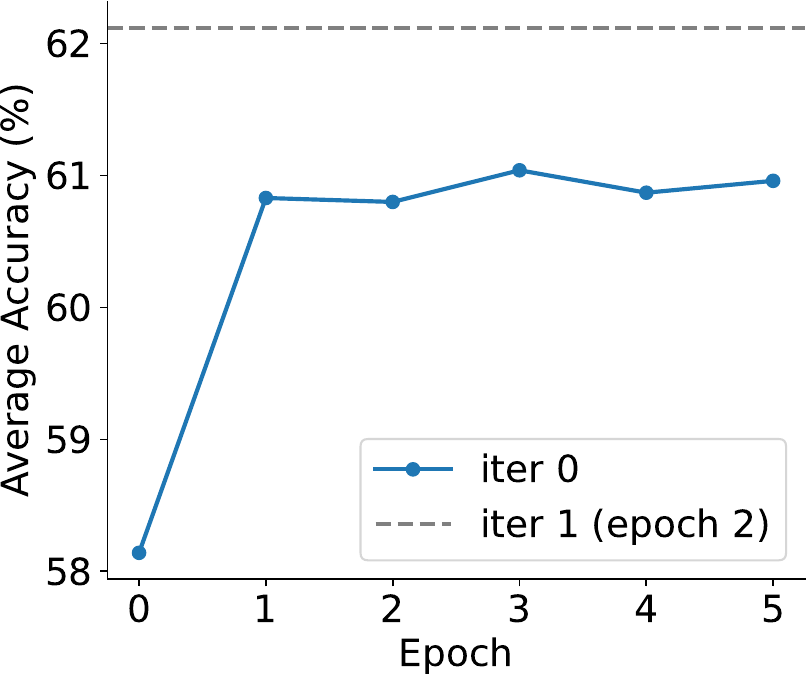}}
\caption{The $\method$ training dynamics of \texttt{zephyr-7b-sft-full} on the 50k synthetic data with regard to the number of training epochs during iteration 0. We can observe that iterative training is pivotal as training for more epochs during iteration 0 reaches a limit and cannot surpass iteration 1.}
\label{fig:ablation_epoch}
\end{figure*}

\subsection{Ablation Studies}
In this subsection, we examine the effect of synthetic dataset size and training epochs within an iteration. 
Our analysis demonstrates the effectiveness of the synthetic data used by $\method$ compared to the SFT data, as well as the necessity of iterative training in $\method$. 
In Appendix~\ref{app:exp}, we present assessment of $\method$ on additional benchmark tasks. 
\begin{figure}
    \centering
    \includegraphics[width=0.4\textwidth]{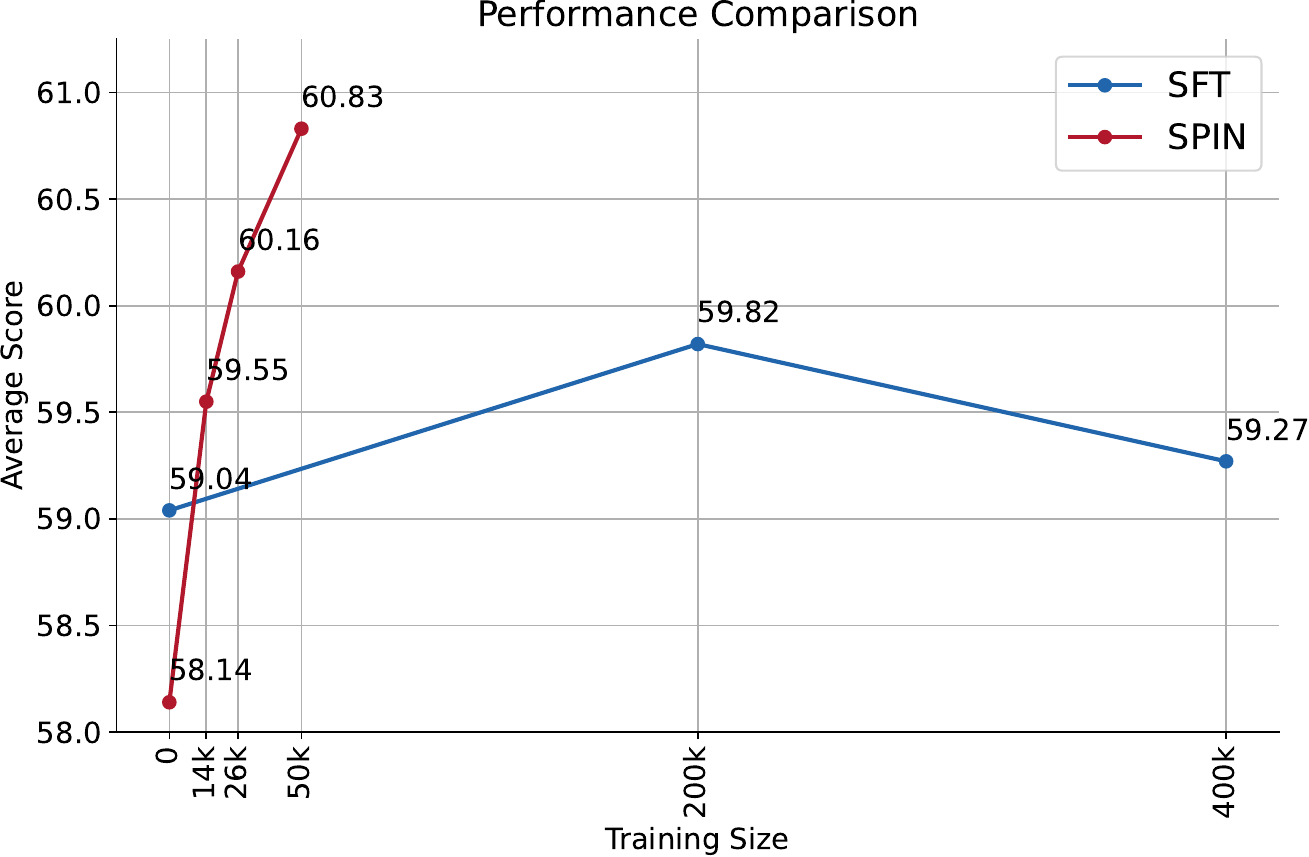}
    \caption{The scaling effect of training size of $\method$ compared to SFT on the average score of Open LLM Leaderboard. For $\method$, we consider training data of sizes $14$k, $26$k and $50$k where the larger dataset contains the smaller dataset. The starting point for $\method$ (with x-axis 0) is the \texttt{zephyr-7b-sft-full} checkpoint, which has been fine-tuned on Ultrachat200k for 1 epoch. We report the model performance trained for 1 epoch with $\method$ on the varying sizes of dataset. We additionally compare with SFT, where we fine-tune Mistral-7B on Ultrachat200k for 3 consecutive epochs and report the model performance at the first epoch as the starting point (with x-axis 0).}
\label{fig:training_size}
\end{figure}

\noindent \textbf{Training Size.}
We investigate the effect of varying training data size on the performance of $\method$. 
In Figure~\ref{fig:training_size}, we demonstrate the effect of training size for $\method$ during iteration $0$ and additionally compare with SFT with the full original dataset. 
Specifically, for the SFT baseline, we fully fine-tune Mistral-7B on Ultrachat200k for three epochs and report first epoch performance as the starting point (with x-axis 0) in the figure for SFT. 
For $\method$, we report the \texttt{zephyr-7b-sft-full} checkpoint as the starting point, which has also been fine-tuned on Ultrachat200k for one epoch. 
We select the training size of $\method$ at iteration 0 to be 14k, 26k, and 50k and generate the data accordingly, ensuring that the larger dataset encompasses the smaller dataset. 
The performance of $\method$ was then evaluated after 1 epoch of self-play fine-tuning for each training size. 
We can observe that, while $\method$ results in notable improvement with increasing training sizes, SFT on further epochs 2 and 3 fails to yield more than $1\%$ improvement. 
Additional results are deferred to Appendix~\ref{app:exp}.

\noindent \textbf{Iterative Training v.s. Training for More Epochs.} 
We further study the training within iteration $0$ and compare with the performance achieved in iteration $1$, particularly contrasting the test performance obtained from extended training duration with that from next iteration. 
Figure~\ref{fig:ablation_epoch} depicts the performance trajectory of the model trained using $\method$ over multiple epochs at iteration 0. 
It is evident that the most substantial improvement occurs during the first two epochs, followed by only modest gains in subsequent epochs. 
Notably, $\method$ exhibits robustness and stability; extending the training duration does not diminish performance but rather maintains a rather consistent level. 
Nevertheless, the observation suggests an inherent limitation to the performance achievable within a single iteration, thereby underscoring the necessity for iterative training. 
As shown by the test performance achieved at iteration 1 in the figures, extending the training in iteration 0 fails to reach the performance comparable to iteration 1.

\section{Conclusion and Discussion} 
This paper introduces a novel fine-tuning method $\method$, to convert a weak LLM to a strong LLM by unleashing the full power of human-annotated data. 
Central to this method is a self-play mechanism, wherein a main player (the LLM) is fine-tuned to differentiate the responses of opponent player (the LLM from previous iteration) from the target data distribution, and the LLM is iteratively aligned with the target data distribution. 
Therefore, $\method$ facilitates the LLM's iterative self-evaluation and enhancement through self-play.
In comparison to supervised fine-tuning and RL fine-tuning methods, $\method$ enables the LLM to self-improve without additional human data or feedback from stronger LLMs. Empirical results demonstrate that $\method$ significantly enhances LLM performance across diverse benchmarks, even outperforming models trained with additional human data or AI feedback.

\noindent \textbf{Limitation and Future Work.} Our theoretical results demonstrate that the optimization process of $\method$ converges if and only if the LLM's distribution aligns with $p_{\mathrm{data}}$. 
Therefore, our study focuses on a fixed target data distribution generated by humans, which inherently imposes a ceiling on the performance of fine-tuned LLM. 
Exploring the dynamically changing target data distribution is an important direction to overcome this limitation and elevate the LLM's performance beyond this ceiling or even to a super-human level. 
Moreover, considering the resource demands of synthetic data generation, another promising avenue for further exploration is to reduce the volume of required synthetic data.

\section*{Acknowledgement}
We thank the anonymous reviewers and area chair for their helpful comments. ZC, YD, HY, KJ, and QG are supported in part by the National Science Foundation CAREER Award 1906169, IIS-2008981, CHE-2247426 and the Sloan Research Fellowship. The views and conclusions contained in this paper are those of the authors and should not be interpreted as representing any funding agencies.

 \section*{Impact Statement}

This paper presents work whose goal is to advance the field of Large Language Models. 
We believe that our work contribute meaningfully to the field, specifically on leveraging synthetic data to enhance LLM without the requirement for human preference annotations. 
The synthetic data generated by $\method$ may be used to further augment the training of various language models.
Moreover, $\method$ demonstrated a substantial improvement in LLMs' capabilities, opening new avenues for their application in various downstream tasks. 
This advancement underscores the transformative potential of LLM fine-tuning in both technological and societal contexts.


\bibliography{deeplearningreference, selftraining}
\bibliographystyle{icml2024}

\newpage
\appendix
\onecolumn
\section{Further Related Work}
\noindent \textbf{Curriculum Learning.} 
In deep learning, it has been observed that training models using data samples arranged in a strategically meaningful order can lead to improved performance compared to training on randomly shuffled data. 
This approach is commonly known as curriculum learning~\citep{bengio2009curriculum,soviany2022curriculum}. 
Initial studies in curriculum learning introduced efficient algorithms that adhere to an `easy-to-hard' progression \citep{valentin2009babysteps,kumar2010self,lee2011learning,zhang2015self}. 
In the field of Natural Language Processing (NLP), criteria such as sentence length and term frequency are commonly utilized \citep{cirik2016visualizing,zhang2018empirical,liu2018curriculum}. 
More recent developments include the application of curriculum learning algorithms in multi-modal learning \citep{liu2021competence,wu2022scaling}.
Our work shares a similar idea to curriculum learning, wherein the training data evolves iteratively—beginning with responses that are easy to distinguish from human-annotated data and gradually progressing to more challenging instances.

\noindent\textbf{Generative Adversarial Networks.} Generative Adversarial Networks (GANs) \citep{goodfellow2014generative} represent a distinct class of generative models, characterized by their unique adversarial process. 
To enhance training stability and data quality, \citet{mao2017least} introduced the Least Squares GAN, employing a least squares loss function for the discriminator. 
A significant advancement in GANs involves the use of Integral Probability Metrics (IPM) \citep{muller1997integral}, particularly highlighted in the development of Wasserstein GAN by \citet{arjovsky2017wasserstein}. 
This model employs IPM in its loss design, enhancing training stability.  
Since then, IPMs have become crucial in GAN design \citep{mroueh2017fisher, gulrajani2017improved}, particularly in constraining the discriminator to a specific function class, thereby preventing it from overpowering the generator. 
Furthermore, \citet{jolicoeur2018relativistic} generalized IPM-based GANs by introducing relativistic discriminator and proposed Relativistic GAN. 
It is worth noting that the objective function defined in our \eqref{eq:f-star} is similar to Relativistic GAN \citep{jolicoeur2018relativistic} and reduces to an IPM framework in Wasserstein GAN \citep{arjovsky2017wasserstein} with a linear loss. However, our approach differs in both the choice of the function class and the training procedure. 
Inspired by GAN, \citet{cheng2023adversarial} proposed an adversarial learning framework named Adversarial Preference Optimization
(APO) that trains the LLM and a reward model in an adversarial game. 
Similarly related to the spirit of our method, Generative Adversarial Imitation Learning (GAIL)~\citep{ho2016generative} was proposed to train separate discriminator and policy networks in each iteration. 
In contrast to the above methods, $\method$ relies on self-play where both the main player and the opponent player are the same LLM from two consecutive iterations.

\noindent\textbf{Alignment with AI Feedback.} 
\CC{The objective of alignment is to fine-tune LLMs to align with human intentions. 
In addition to using human demonstrations, AI feedback is emerging as a crucial component in the alignment process. 
Constitutional AI \citep{bai2022constitutional} leveraged AI feedback to align language models through a combination of both supervised learning and reinforcement learning (RL) phases. 
In the RL phase, \citep{bai2022constitutional} applied Reinforcement Learning from AI Feedback (RLAIF), training a reward model using AI-generated preferences, followed by RL using the reward. 
\citet{lee2023rlaif} demonstrated that AI feedback can achieve comparable or superior performance to human feedback in RL fine-tuning. 
They also demonstrated that RLAIF can improve upon an SFT policy even when the LLM labeler has the same size as the policy. 
\citet{saunders2022self} studied the scaling properties of self-critique and introduced a framework for comparing the critique ability to generation and discrimination ability. 
Self-critique models employ the LLM itself to generate natural language critiques through behavioral cloning, assisting human evaluators. }


\section{Experiments}\label{app:exp}
\subsection{Hyperparameters and Implementation Details}
\begin{table*}[ht]
    \centering
    \caption{Detailed information of HuggingFace Open LLM Leaderboard. For each evaluation dataset, we present the number of few-shot examples and metric adopted for evaluation.}
    \resizebox{0.7\textwidth}{!}{%
    \begin{tabular}{c | c c c c c c c}
    \toprule
        Datasets & Arc & TruthfulQA & Winogrande & GSM8k & HellaSwag & MMLU \\
        \midrule
        \# few-shot & 25 & 0 & 5 & 5 & 10 & 5  \\
        Metric & \texttt{acc\_norm} & \texttt{mc2} & \texttt{acc} & \texttt{acc} & \texttt{acc\_norm} & \texttt{acc}  \\
    \bottomrule
    \end{tabular}%
    }
    \label{tab:open-llm-info}
\end{table*}
We use the Alignment Handbook library~\citep{alignment_handbook2023} as the codebase for our self-play fine-tuning method $\method$, which includes DeepSpeed ZeRO-3~\citep{rajbhandari2020zero} and FlashAttention-2~\citep{dao2023flashattention} to reduce training cost. We train our models with RMSProp~\citep{hinton2012neural} optimizer with no weight decay for all iterations as commonly used in fine-tuning LLMs for alignment, with a global batch size of $64$, $10\%$ warmup steps and bfloat16 precision.  
We set the peak learning rate to be 5e-7 for iterations 0 and 1, and decay this peak learning rate to 1e-7 for iteration 2 and 3 as we are approaching the end of self-play fine-tuning. Lastly, we choose $\beta=0.1$ and max sequence length to be $2048$ tokens as in \citet{alignment_handbook2023}. We note that at the last iteration (iter-3) where the  model is close to convergence, we increase the value of $\beta$ to $5.0$. 
We use the Accelerate library~\citep{accelerate} to generate our synthetic data using distributed inference with multiple GPUs with a global batch size of $64$. We consider the prompting template ``\#\#\# Instruction: \{prompt\}\textbackslash n\textbackslash n\#\#\# Response: '' as commonly used in~\citet{alpaca}. For Ultrachat200k containing multi-round conversations, we only sample the first round as our prompt and ground truth completion pairs.

\subsection{Training Overhead}
\CC{The cost overhead introduced by SPIN is mostly the generation of synthetic data from the LLM that we train. The cost of the fine-tuning process remains computationally equal to that of SFT and DPO. We report both the generation and training time for SPIN in Table~\ref{tab:times} . Results were obtained using a machine with 8xA100 (80G) GPUs. For per 64 examples, the generation time and training time are 6.69s and 10s respectively.}

\begin{table}[h!]
    \caption{Generation and Training Times for Different Iterations}
    \centering
    \begin{tabular}{c|cc|cc|cc|cc}
       \toprule
       Iteration & \multicolumn{2}{c|}{Iter 0} & \multicolumn{2}{c|}{Iter 1} & \multicolumn{2}{c|}{Iter 2} & \multicolumn{2}{c}{Iter 3} \\ \midrule
          Process& Generation  & Training  & Generation  & Training  & Generation  & Training  & Generation  & Training  \\  \midrule
        Time  & 1.45h & 4.32h & 1.45h & 8.64h & 1.45h & 8.64h & 1.45h & 8.64h \\     \bottomrule
    \end{tabular}
    \label{tab:times}
\end{table}

\CC{It is evident that the generation time is dominated by the training time at each iteration. The estimated time in Table~\ref{tab:times} is based on the fact that we generate 50k examples per iteration. Please note that the doubled training time from iter 1 to iter 3 is attributed to the utilization of a double-sized dataset (the combination of 50k synthetic data from the previous iteration and 50k synthetic data in the current iteration), as discussed in our Section~\ref{sec:exp}.}

\subsection{Additional Experiment Result for SPIN+DPO} 
\CC{\method{} requires only the SFT data to improve over the traditional SFT stage and can sit between SFT and RL finetuning. Suppose additional preference data is provided, we can use the additional data to further improve the performance of the model after SPIN using RL fine-tuning. }

Starting at \method{} iteration 3, we further train the model with DPO for two epochs on the $62$k preference data from the UltraFeedback Binarized dataset~\citep{cui2023ultrafeedback}, which consists of both chosen and rejected responses evaluated by GPT-4. Detailed performances are presented in Table~\ref{tab:additionspindpo}.

\begin{table*}[ht]
    \centering
    \caption{Performance of $\method$ + DPO based on \texttt{zephyr-7b-sft-full} across HuggingFace Open LLM Leaderboard datasets, compared with all baselines. We also denote the average improvement over last iteration in the Average column.}
    \resizebox{0.9\textwidth}{!}{%
    \begin{tabular}{c | c c c c c c c}
    \toprule
        Model & Arc & TruthfulQA & Winogrande & GSM8k & HellaSwag & MMLU & Average \\
        \midrule
        \texttt{zephyr-7b-dpo-full} & 63.65 & 55.19 & 72.61 &  33.43 & 84.44 & 58.52 & 61.31 \\
        \midrule
        \texttt{zephyr-7b-sft-full} & 60.41 & 43.73 & 74.19 & 26.76 & 82.85 & 60.92 & 58.14 \\
        $\method$ iteration $0$ & 63.40 & 49.18 & 72.69 & 35.10 & 84.38 & 60.03 & $60.80_{\textcolor{ao}{(+2.66)}}$ \\
        $\method$ iteration $1$ & 65.19 & 55.17 & 72.30 & 35.78 & 84.96 & 59.34 & $62.12_{\textcolor{ao}{(+1.32)}}$ \\
        $\method$ iteration $2$ & 65.96	& 54.91 & 73.56 & 38.06 & 85.41 & 59.93 & $62.97_{\textcolor{ao}{(+0.85)}}$ \\
        $\method$ iteration $3$ & 65.87 & 54.90 & 73.72 & 38.97 & 85.54 & 59.99 & $63.16_{\textcolor{ao}{(+0.19)}}$ \\
        \rowcolor{LightCyan} $\method$ iteration $3$ + DPO& 66.47 & 60.07 & 78.06 & 37.98 & 86.17 & 59.68 & $64.05_{\textcolor{ao}{(+0.89)}}$ \\
    \bottomrule
    \end{tabular}%
    }
    \label{tab:additionspindpo}
\end{table*}

We can observe that the checkpoint trained by $\method$ can be further improved using DPO, yielding an enhancement of $0.89\%$ on average. Notably, the improvement is particularly significant on the TruthfulQA benchmark with around $5\%$.

\subsection{Further Experiment Results}
In Table~\ref{tab:main}, we show the detailed performance of $\method$ at different iterations on each of the task in Open LLM Leaderboard. In Table~\ref{tab:ablation_data}, we also show the performance of SFT from \texttt{zephyr-7b-sft-full} on Ultrachat200k for one epoch. 
While self-play fine-tuning with synthetic data from \texttt{zephyr-7b-sft-full} effectively improves its performance, simply fine-tuning it again on the SFT data leads to degraded performance, as similarly observed in Figure~\ref{fig:training_size}. 

\begin{table*}[ht]
    \centering
    \caption{Test performance of $\method$ based on \texttt{zephyr-7b-sft-full} across HuggingFace Open LLM Leaderboard datasets. We also denote the average improvement over last iteration in the Average column.}
    \resizebox{0.9\textwidth}{!}{%
    \begin{tabular}{c | c c c c c c c}
    \toprule
        Model & Arc & TruthfulQA & Winogrande & GSM8k & HellaSwag & MMLU & Average \\
        \midrule
        \texttt{zephyr-7b-sft-full} & 60.41 & 43.73 & 74.19 & 26.76 & 82.85 & 60.92 & 58.14 \\
        $\method$ iteration $0$ & 63.40 & 49.18 & 72.69 & 35.10 & 84.38 & 60.03 & $60.80_{\textcolor{ao}{(+2.66)}}$ \\
        $\method$ iteration $1$ & 65.19 & 55.17 & 72.30 & 35.78 & 84.96 & 59.34 & $62.12_{\textcolor{ao}{(+1.32)}}$ \\
        $\method$ iteration $2$ & 65.96	& 54.91 & 73.56 & 38.06 & 85.41 & 59.93 & $62.97_{\textcolor{ao}{(+0.85)}}$ \\
        $\method$ iteration $3$ & 65.87 & 54.90 & 73.72 & 38.97 & 85.54 & 59.99 & $63.16_{\textcolor{ao}{(+0.19)}}$ \\
    \bottomrule
    \end{tabular}%
    }
    \label{tab:main}
\end{table*}
\begin{table*}[ht]
    \centering
    \caption{Test performance of \texttt{zephyr-7b-sft-full} fine-tuned on Ultrachat200k for 1 more epoch across HuggingFace Open LLM benchmark datasets. SFT fails to further leverage the fine-tuning data for performance enhancement and even results in degraded performance.}
    \resizebox{0.78\textwidth}{!}{%
    \begin{tabular}{c | c c c c c c c}
    \toprule
        Model & Arc & TruthfulQA & Winogrande & GSM8k & HellaSwag & MMLU & Average \\
        \midrule
        \texttt{zephyr-7b-sft-full} & 60.41 & 43.73 & 74.19 & 26.76 & 82.85 & 60.92 & 58.14 \\
        \rowcolor{red!20}SFT epoch 1 & \textcolor{red}{57.76} & 44.39 & 75.77 & 25.85 & 81.69 & \textcolor{red}{57.89} & \textcolor{red}{57.23} \\
    \bottomrule
    \end{tabular}%
    }
    \label{tab:ablation_data}
\end{table*}

\noindent \textbf{Further Investigation on More Tasks.}
Here, we further investigate the performance of $\method$ on a broader variety of tasks, including MT-Bench~\citep{zheng2023judging}, Big-Bench~\citep{srivastava2023beyond} and OpenBookQA~\citep{mihaylov2018can} in addition to the Open LLM Leaderboard tasks. 
Specifically, we use the following tasks from Big-Bench-Hard for a more comprehensive evaluation, including Causal Judgment (causal reasoning), Sports Understanding (commonsense reasoning) and Formal Fallacies (logical reasoning). 
In Table~\ref{tab:ablation_bbh}, we show the resulting scores of $\method$ on MT-Bench as well as those tasks from Big-Bench. 
In Figure~\ref{fig:mtbench}, we detail the model performances on MT-Bench with regard to different types of questions. 
We can see a notably robust improvement in the performance of $\method$ on various tasks besides the HuggingFace Benchmark, without major degradation. Notably, on MT-Bench, the model fine-tuned by $\method$ has surpassed the performance of \texttt{vicuna-13b-v1.5}~\citep{vicuna2023} with a score of $6.57$. 
\begin{table*}[ht]
    \centering
    \caption{Test performance on other reasoning benchmark datasets for $\method$ at different iterations and \texttt{zephyr-7b-sft-full}. We report the average score for MT-Bench and the accuracy score for Big Bench datasets under standard few-shot CoT evaluation. On OpenBookQA, we report \texttt{acc\_norm} with 1-shot example as used in~\citet{anil2023palm}. As similar to Open LLM Leaderboard evaluation, we observe a steady improvement in performance on the other benchmark tasks, with no significant degradation.}
    \resizebox{0.8\textwidth}{!}{%
    \begin{tabular}{c | c | c c c | c}
    \toprule
        Model & MT-Bench & BB-causal & BB-formal & BB-sports & OpenBookQA \\
        \midrule
        \texttt{zephyr-7b-sft-full} & 5.94 & 56.15 & 49.6 & 96.0 & 45.4\\
        $\method$ iteration $0$ & $6.46_{\textcolor{ao}{(+0.52)}}$ & 57.75 & 51.6 & 95.2 & 46.8\\
        $\method$ iteration $1$ & $6.65_{\textcolor{ao}{(+0.19)}}$ & 58.82 & 51.2 & 95.2 & 47.2\\
        $\method$ iteration $2$ & $6.78_{\textcolor{ao}{(+0.13)}}$ & 59.36 & 51.2 & 94.4 & 47.6\\
    \bottomrule
    \end{tabular}%
    }
    \label{tab:ablation_bbh}
\end{table*}
\begin{figure}
    \centering
    \includegraphics[width=0.45\textwidth]{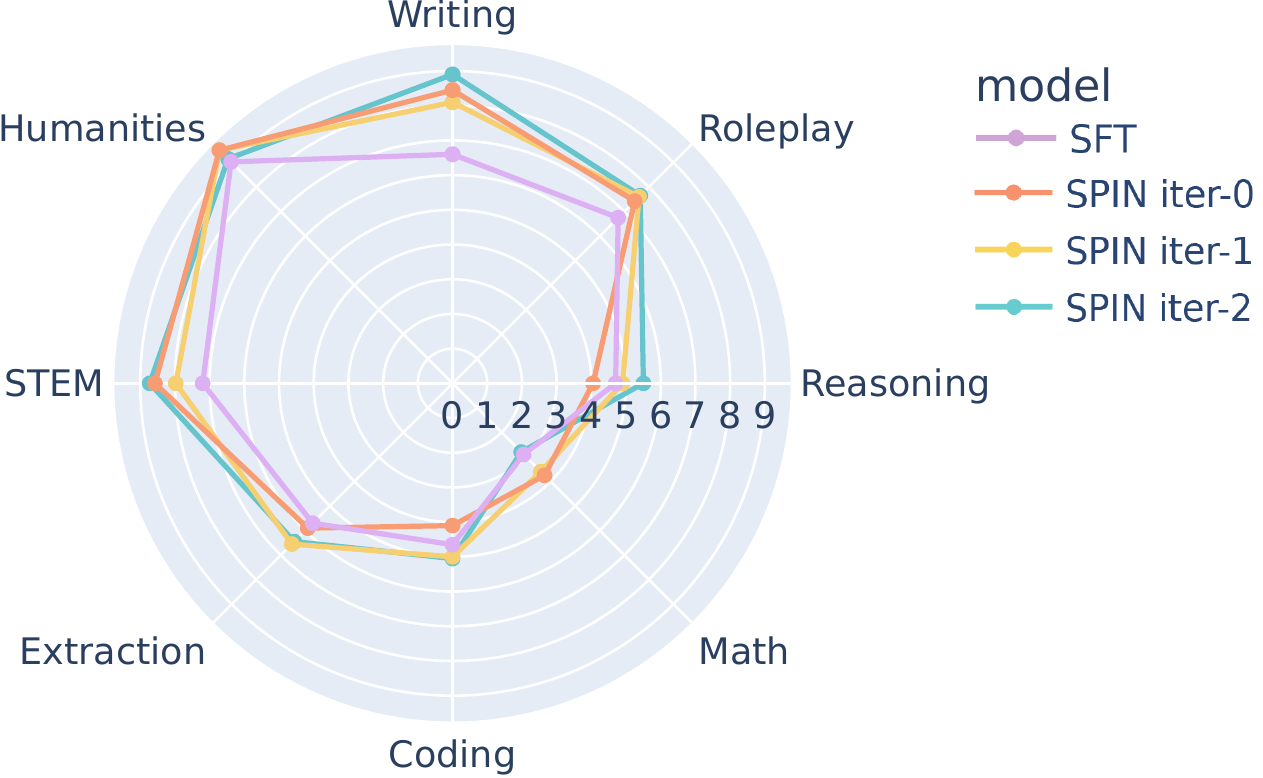}
    \caption{Model performance on MT-Bench. We compare $\method$ across different iterations with the base SFT model. Starting from iteration 1, our fine-tuned model by $\method$ robustly outperforms the SFT checkpoint on all evaluation aspects.}
    \label{fig:mtbench}
\end{figure}

\subsection{Generation Examples}
In Tables~\ref{tab:app_exp1} and~\ref{tab:app_exp2}, we further provide the generation examples of our fine-tuned model by $\method$ from different iterations. We can observe an improvement in response quality as compared to the generation of the SFT checkpoint. Meanwhile, the model generations at higher iterations typically becomes more concise than iteration $0$ and resemble the ground truth completion better. 

\begin{table}[h!]
\caption{Generation example of our fine-tuned model by $\method$ at different iterations.}
\begin{tabularx}{\textwidth}{l|X}
\toprule
Prompt & How does the economic growth of Horsham compare to nearby towns? \\
\midrule
Ground truth & I do not have access to the latest economic data on horsham and nearby towns. However, you can check the latest economic statistics and reports provided by the local government or the official websites of horsham and nearby towns for a detailed comparison of their economic growth. \\
\midrule
\texttt{zephyr-7b=sft-full} & According to the given material, Horsham was growing "significantly" faster than nearby towns, indicating that it was economically thriving.\textbackslash n However, the given material does not specify the exact rates or levels of economic growth of Horsham compared to nearby towns. Without this information, it is challenging to determine whether the differences in economic growth between Horsham and nearby towns were significant or minimal.\\
\midrule
Iteration 0 & I do not have access to the latest economic data for horsham and nearby towns. However, horsham is known to be a prosperous town in west sussex, uk, with a diverse economy that includes agriculture, manufacturing, and services. The town has a strong retail sector, and it is also home to several large companies, including the global pharmaceutical company novartis. In terms of comparison with nearby towns, horsham is generally considered to be one of the more affluent towns in the region. However, the economic growth of horsham and its neighboring towns may vary depending on various factors such as industry, investment, and government policies.\\
\midrule
Iteration 1 & I do not have access to real-time data or information about the most recent economic growth of horsesham and nearby towns. However, according to the recent reports from various sources such as the office for national statistics, horsesham's economy primarily relies on retail, healthcare, education, and public administration industries. Compared to nearby towns such as lewes, chichester, and bognor regis, horsesham's economy may have a different composition and growth rate based on their specific industries and economic policies.\\
\midrule
Iteration 2 & I don't have access to the most recent economic data. However, according to a 2019 report by pwc, horsham’s expected economic growth up to 2025 was projected to be around $2.3\%$ per annum, which is slightly below the average for the south east of england ($2.8\%$). Neighbouring towns like crawley and bognor regis were projected to experience similar growth rates. note: this may not reflect the impact of the covid-19 pandemic that occurred in 2020.\\
\bottomrule
\end{tabularx}
\label{tab:app_exp1}
\end{table}

\begin{table}[h!]
\caption{Another generation example of our fine-tuned model by $\method$ at different iterations.}
\begin{tabularx}{\textwidth}{l|X}
\toprule
Prompt & Can you clarify if Poppelsdorfer Allee 45 and the old site of the Mathematical Institute are the same location? Answer according to: Go straight to Poppelsdorfer Allee 45 and sign-in.\textbackslash n There is a big lecture hall at the old site of the Mathematical Institute (Wegelerstr. 10, Bonn). \\
\midrule
Ground truth & No, Poppelsdorfer Allee 45 and the old site of the Mathematical Institute (Wegelerstr. 10, Bonn) are not the same location. To sign-in, one should go straight to Poppelsdorfer Allee 45. However, there is a big lecture hall at the old site of the Mathematical Institute. \\
\midrule
\texttt{zephyr-7b=sft-full} & Yes, Poppelsdorfer Allee 45 and the old site of the Mathematical Institute are the same location.\\
\midrule
Iteration 0 & Yes, Poppelsdorfer Allee 45 and the old site of the Mathematical Institute (Wegelerstr. 10, Bonn) are the same location. The Mathematical Institute has been relocated and renamed as the Institute for Mathematics and Theoretical Physics, but the building where it used to be is still referred to as the old site. The big lecture hall at the old site still exists and may be used for events or classes.\\
\midrule
Iteration 1 & No, Poppelsdorfer Allee 45 and the old site of the Mathematical Institute are not the same location. While Poppelsdorfer Allee 45 is where attendees should sign-in, the old site of the Mathematical Institute (located at Wegelerstr. 10, Bonn) has a large lecture hall mentioned in the statement. Therefore, it can be assumed that the two locations are separate and not the same.\\
\midrule
Iteration 2 & Poppelsdorfer Allee 45 and the old site of the Mathematical Institute (Wegelerstr. 10, Bonn) are two different locations. As per the instruction, you need to go straight to Poppelsdorfer Allee 45 and sign-in. However, there is a lecture hall mentioned at the old site of the Mathematical Institute (Wegelerstr. 10, Bonn).\\
\bottomrule
\end{tabularx}
\label{tab:app_exp2}
\end{table}
\newpage
\section{Proof of Theorems in Section \ref{sec:thm} }\label{sec:proof_thm}
\subsection{Proof of Theorem~\ref{thm:stop}}
\begin{proof}[Proof of Theorem~\ref{thm:stop}]
To begin with, we prove the ``Sufficiency'' in Theorem~\ref{thm:stop}. Since $p_{\mathrm{data}}(\cdot|\xb)= p_{\btheta_t}(\cdot|\xb)$, by symmetry property of $\yb$ and $\yb'$, we have for any $\btheta \in \bTheta$ that 
\begin{align*}
2L_{\method}(\btheta, \btheta_t) &= \EE_{\xb \sim q(\cdot), \yb \sim p_{\mathrm{data}}(\cdot|\xb), \yb' \sim p_{\btheta_t}(\cdot|\xb)}\bigg[\ell\bigg(\gamma \log \frac{p_{\btheta}(\yb | \xb)}{p_{\btheta_t}(\yb | \xb)}-\gamma \log \frac{p_{\btheta}(\yb' | \xb)}{p_{\btheta_t}(\yb' | \xb)}\bigg)\bigg] \\
&\qquad+ \EE_{\xb \sim q(\cdot), \yb' \sim p_{\mathrm{data}}(\cdot|\xb), \yb \sim p_{\btheta_t}(\cdot|\xb)}\bigg[\ell\bigg(\gamma \log \frac{p_{\btheta}(\yb | \xb)}{p_{\btheta_t}(\yb | \xb)}-\gamma \log \frac{p_{\btheta}(\yb' | \xb)}{p_{\btheta_t}(\yb' | \xb)}\bigg)\bigg]\\
&= \EE_{\xb \sim q(\cdot), \yb \sim p_{\mathrm{data}}(\cdot|\xb), \yb' \sim p_{\btheta_t}(\cdot|\xb)}\bigg[\ell\bigg(\gamma \log \frac{p_{\btheta}(\yb | \xb)}{p_{\btheta_t}(\yb | \xb)}-\gamma \log \frac{p_{\btheta}(\yb' | \xb)}{p_{\btheta_t}(\yb' | \xb)}\bigg)\\
&\qquad + \ell\bigg(\gamma \log \frac{p_{\btheta}(\yb' | \xb)}{p_{\btheta_t}(\yb' | \xb)}-\gamma \log \frac{p_{\btheta}(\yb | \xb)}{p_{\btheta_t}(\yb | \xb)}\bigg)\bigg]\\
&\geq 2\EE_{\xb \sim q(\cdot), \yb \sim p_{\mathrm{data}}(\cdot|\xb), \yb' \sim p_{\btheta_t}(\cdot|\xb)}\bigg[\ell\bigg(\frac{\gamma}{2} \log \frac{p_{\btheta}(\yb | \xb)}{p_{\btheta_t}(\yb | \xb)}-\frac{\gamma}{2}\log \frac{p_{\btheta}(\yb' | \xb)}{p_{\btheta_t}(\yb' | \xb)}\\
&\qquad + \frac{\gamma}{2} \log \frac{p_{\btheta}(\yb' | \xb)}{p_{\btheta_t}(\yb' | \xb)}-\frac{\gamma}{2} \log \frac{p_{\btheta}(\yb | \xb)}{p_{\btheta_t}(\yb | \xb)}\bigg)\bigg]\\
&= 2\ell(0),
\end{align*}
where the inequality is due to Jensen's inequality (recalling that $\ell$ is convex in Assumption \ref{assm:1}). Therefore, we have that $L_{\method}(\btheta, \btheta_t) \geq \ell(0) = L_{\method}(\btheta_t, \btheta_t)$, which means that $\btheta_t$ is the global optimum of \eqref{eq:loss}. As a consequence, the gradient at the point $\btheta_t$ is zero, which concludes $\btheta_{t+1} = \btheta_{t}$.  

Next, we prove the ``Necessity''. Define $g(\lambda)$ as follows: 
\begin{align*}
g(\lambda) = \EE_{\xb \sim q(\cdot), \yb \sim p_{\mathrm{data}}(\cdot|\xb), \yb' \sim p_{\btheta_t}(\cdot|\xb)}\bigg[\ell\bigg(\lambda \log \frac{p_{\mathrm{data}}(\yb | \xb)}{p_{\btheta_t}(\yb | \xb)}-\lambda \log \frac{p_{\mathrm{data}}(\yb' | \xb)}{p_{\btheta_t}(\yb' | \xb)}\bigg)\bigg].\label{eq:}   
\end{align*}
Then we have $g(0) = \ell(0)$ and 
\begin{align*}
g'(0) &= \EE_{\xb \sim q(\cdot), \yb \sim p_{\mathrm{data}}(\cdot|\xb), \yb' \sim p_{\btheta_t}(\cdot|\xb)}\bigg[\ell'(0)\bigg(\log \frac{p_{\mathrm{data}}(\yb | \xb)}{p_{\btheta_t}(\yb | \xb)}-  \log \frac{p_{\mathrm{data}}(\yb' | \xb)}{p_{\btheta_t}(\yb' | \xb)}\bigg)\bigg]\\
&= \ell'(0) \bigg( \EE_{\xb \sim q(\cdot), \yb \sim p_{\mathrm{data}}(\cdot|\xb)}\bigg[\log \frac{p_{\mathrm{data}}(\yb | \xb)}{p_{\btheta_t}(\yb | \xb)}\bigg] - \EE_{\xb \sim q(\cdot), \yb' \sim p_{\btheta_t}(\cdot|\xb)}\bigg[\log \frac{p_{\mathrm{data}}(\yb' | \xb)}{p_{\btheta_t}(\yb' | \xb)}\bigg] \bigg)\\
&= \ell'(0)\Big[\mathrm{KL}\big(p_{\mathrm{data}}(\cdot|\xb)\big|\big|p_{\btheta_t}(\cdot|\xb)\big) + \mathrm{KL}\big(p_{\btheta_t}(\cdot|\xb)\big|\big|p_{\mathrm{data}}(\cdot|\xb)\big)\Big] \\
&< 0,
\end{align*}
where the last inequality is due to the condition that $\ell'(0) < 0$. Therefore, there exist a $\lambda_0$ such that for all $0 < \lambda < \lambda_0$, we have $g(\lambda) < \ell(0)$. Choose $\btheta^{*}$ such that $p_{\btheta^{*}}(\yb|\xb) = p_{\mathrm{data}}(\yb|\xb)$. For those $0 < \lambda < \lambda_0$, we have that 
\begin{align*}
L_{\method}(\btheta^{*}, \btheta_{t}) &= \EE_{\xb \sim q(\cdot), \yb \sim p_{\btheta^{*}}(\cdot|\xb), \yb' \sim p_{\btheta_t}(\cdot|\xb)}\bigg[\ell\bigg(\lambda \log \frac{p_{\btheta^{*}}(\yb | \xb)}{p_{\btheta_t}(\yb | \xb)}-\lambda \log \frac{p_{\btheta^{*}}(\yb' | \xb)}{p_{\btheta_t}(\yb' | \xb)}\bigg)\bigg]\\
&=\EE_{\xb \sim q(\cdot), \yb \sim p_{\mathrm{data}}(\cdot|\xb), \yb' \sim p_{\btheta_t}(\cdot|\xb)}\bigg[\ell\bigg(\lambda \log \frac{p_{\mathrm{data}}(\yb | \xb)}{p_{\btheta_t}(\yb | \xb)}-\lambda \log \frac{p_{\mathrm{data}}(\yb' | \xb)}{p_{\btheta_t}(\yb' | \xb)}\bigg)\bigg]\\
& = g(\lambda) \\
& < g(0)\\
&= L_{\method}(\btheta_{t}, \btheta_{t}),
\end{align*}
where the second equality holds by the choice of $p_{\btheta^{*}}(\cdot|\xb)$, and the inequality holds due to the choice of $\lambda$. Therefore, we conclude that $\btheta_t$ is not the global optimum of \eqref{eq:loss} if $p_{\btheta_t}(\cdot|\xb)  \not= p_{\mathrm{data}}(\cdot|\xb)$.
\end{proof}

\subsection{Proof Theorem~\ref{thm:main2}}
We need the following auxiliary lemma before we prove Theorem~\ref{thm:main2}.
\begin{lemma}\label{lm:aux}
Suppose that $\ell(t) = \log(1 + \exp(-t))$ and for $a, b > 0$, the following inequality holds
\begin{align*}
 a\ell(t) + b\ell(-t) \geq a\log(1+b/a) + b\log(1 + a/b),     
\end{align*}
the equality holds if and only if $t = \log(a/b)$. 
\end{lemma}
\begin{proof}[Proof of Lemma~\ref{lm:aux}]
Define $g(t) =  a\ell(t) + b\ell(-t) =  a\log(1+\exp(-t)) + b\log(1+\exp(t))$, then we have 
\begin{align*}
g'(t) = -\frac{a\exp(-t)}{1+\exp(-t)} +   \frac{b\exp(t)}{1+\exp(t)} = \frac{-a+b\exp(t)}{1+\exp(t)}.
\end{align*}
Therefore, $g'(t) < 0$ when $t < \log(a/b)$, $g'(t) > 0$ when $t > \log(a/b)$, which indicates that $g$ achieves it minimum at $t = \log(a/b)$ which concludes the proof.
\end{proof}

Lemma~\ref{lm:aux} shows that the global minimum of $a\ell(t) + b\ell(-t)$ is achieved when $t = \log(a/b)$. Based on Lemma~\ref{lm:aux}, we can further prove that \eqref{eq:f-star} with the logistic loss function has a closed-form solution if we ignore the constraint set $\cF_{t}$. 

\begin{lemma}\label{thm:closed-form solution}
Denote $p_{+}(\yb, \yb', \xb) = q(\xb)\cdot p_{\mathrm{data}}(\yb|\xb)\cdot p_{\btheta_t}(\yb'|\xb)$ and $p_{-}(\yb, \yb', \xb) = q(\xb)\cdot p_{\btheta_t}(\yb'|\xb) \cdot p_{\mathrm{data}}(\yb|\xb)$, 
\begin{align*}
\EE_{\xb\sim q(\cdot), \yb\sim p_{\mathrm{data}}(\cdot|\xb), y'\sim p_{\btheta_t}(\cdot|\xb)}\big[ \ell\big(f(\xb, \yb) - f(\xb, \yb')\big) 
\big] \geq \log2 - \mathrm{JSD}(p_{+}\|p_{-}),
\end{align*}
where $\mathrm{JSD}(p_{+}\|p_{-})$ represents the Jensen–Shannon divergence which is defined as follows
\begin{equation*}
    \mathrm{JSD}\Big(p \Big\|q\Big) = \frac{1}{2}\mathrm{KL}\Big(p \Big\|\frac{p+q}{2}\Big) +  \frac{1}{2}\mathrm{KL}\Big(q \Big\|\frac{p+q}{2}\Big),
\end{equation*}
where $\mathrm{KL}(\cdot\|\cdot)$ is KL-divergence. $\mathrm{JSD}$ is always non-negative and equals zero if and only if $p_{+}$ and $p_{-}$ are identical. Moreover, the global minimum value $\log 2 - \mathrm{JSD}(p_{+}\|p_{-})$ is achieved by $f^{*}$ if and only if, 
\begin{align*}
f^{*}(\xb, \yb) =   Z(\xb) + \log\bigg(\frac{p_{\mathrm{data}}(\yb|\xb)}{p_{\btheta_{t}}(\yb|\xb)}\bigg),   
\end{align*}
where $Z(\xb)$ is any function that is possibly dependent on $\xb$. 
\end{lemma}

\begin{proof}[Proof of Lemma~\ref{thm:closed-form solution}]
We rewrite the objective function in the following formula, 
\begin{align*}
&2\EE_{\xb\sim q(\cdot), \yb\sim p_{\mathrm{data}}(\cdot|\xb), \yb'\sim p_{\btheta_t}(\cdot|\xb)}\big[ \ell\big(f(\xb, \yb) - f(\xb, \yb')\big) 
\big] \notag\\
&= \int q(\xb)p_{\mathrm{data}}(\yb|\xb)p_{\btheta_{t}}(\yb'|\xb)\big[ \ell\big(f(\xb, \yb) - f(\xb, \yb')\big) 
\big]d\yb d\yb' \notag\\
&\qquad + \int q(\xb)p_{\mathrm{data}}(\yb'|\xb)p_{\btheta_{t}}(\yb|\xb)\big[ \ell\big(f(\xb, \yb') - f(\xb, \yb)\big) 
\big]d\yb d\yb' \notag\\
&= \int q(\xb)p_{\mathrm{data}}(\yb|\xb)p_{\btheta_{t}}(\yb'|\xb) \ell\big(f(\xb, \yb) - f(\xb, \yb')\big) 
 \notag\\
&\qquad + q(\xb)p_{\mathrm{data}}(\yb'|\xb)p_{\btheta_{t}}(\yb|\xb) \ell\big(f(\xb, \yb') - f(\xb, \yb)\big) 
d\yb d\yb' \notag\\
&\overset{(i)}{\geq} \int q(\xb)p_{\mathrm{data}}(\yb|\xb)p_{\btheta_{t}}(\yb'|\xb) \log\bigg(1 + \frac{p_{\mathrm{data}}(\yb'|\xb)p_{\btheta_{t}}(\yb|\xb)}{p_{\mathrm{data}}(\yb|\xb)p_{\btheta_{t}}(\yb'|\xb) }\bigg) 
 \notag\\
&\qquad + q(\xb)p_{\mathrm{data}}(\yb'|\xb)p_{\btheta_{t}}(\yb|\xb) \log\bigg(1 + \frac{p_{\mathrm{data}}(\yb|\xb)p_{\btheta_{t}}(\yb'|\xb)}{p_{\mathrm{data}}(\yb'|\xb)p_{\btheta_{t}}(\yb|\xb)}\bigg)   
d\yb d\yb',
\end{align*}  
where the inequality is due to $ a\ell(t) + b\ell(-t) \geq a\log(1+b/a) + b\log(1+a/b)$ in Lemma~\ref{lm:aux} with $a = q(\xb)p_{\mathrm{data}}(\yb|\xb)p_{\btheta_{t}}(\yb'|\xb), b = q(\xb)p_{\mathrm{data}}(\yb'|\xb)p_{\btheta_{t}}(\yb|\xb)$, $t = f(\xb, \yb) - f(\xb, \yb')$. The equality (i)  holds if and only if the following equation holds almost surely for any $\xb, \yb, \yb'$, 
\begin{align}
&f(\xb, \yb) - f(\xb, \yb') = \log\bigg(\frac{p_{\mathrm{data}}(\yb|\xb)p_{\btheta_{t}}(\yb'|\xb)}{p_{\mathrm{data}}(\yb'|\xb)p_{\btheta_{t}}(\yb|\xb)}\bigg). \label{eq:mid1}
\end{align}
Equation \eqref{eq:mid1} is equivalent to
\begin{align*}
f(\xb, \yb) - \log\bigg(\frac{p_{\mathrm{data}}(\yb|\xb)}{p_{\btheta_{t}}(\yb|\xb)}\bigg) = f(\xb, \yb') - \log\bigg(\frac{p_{\mathrm{data}}(\yb'|\xb)}{p_{\btheta_{t}}(\yb'|\xb)}\bigg)
\end{align*}
holds almost surely for any $\xb, \yb, \yb'$. Therefore, the equality (i) holds if and only if there exists some $Z(\xb)$ such that 
\begin{align*}
f(\xb, \yb) = Z(\xb) + \log\bigg(\frac{p_{\mathrm{data}}(\yb|\xb)}{p_{\btheta_{t}}(\yb|\xb)}\bigg).
\end{align*}
Recall that $p_{+}(\yb, \yb'|\xb) = p_{\mathrm{data}}(\yb|\xb)\cdot p_{\btheta_t}(\yb|\xb)$ and $p_{-}(\yb, \yb'|\xb) = p_{\btheta_t}(\yb|\xb) \cdot p_{\mathrm{data}}(\yb|\xb)$. Then, the right-hand side of (i) can be written as 
\begin{align*}
&\int q(\xb)p_{\mathrm{data}}(\yb|\xb)p_{\btheta_{t}}(\yb'|\xb) \log\bigg(1 + \frac{p_{\mathrm{data}}(\yb'|\xb)p_{\btheta_{t}}(\yb|\xb)}{p_{\mathrm{data}}(\yb|\xb)p_{\btheta_{t}}(\yb'|\xb) }\bigg) 
 \notag\\
&\qquad + q(\xb)p_{\mathrm{data}}(\yb'|\xb)p_{\btheta_{t}}(\yb|\xb) \log\bigg(1 + \frac{p_{\mathrm{data}}(\yb|\xb)p_{\btheta_{t}}(\yb'|\xb)}{p_{\mathrm{data}}(\yb'|\xb)p_{\btheta_{t}}(\yb|\xb)}\bigg)   
d\yb d\yb' \\
&= \int p_{+}(\yb, \yb'|\xb) \log\bigg(1 + \frac{p_{-}(\yb, \yb'|\xb)}{p_{+}(\yb, \yb'|\xb)}\bigg) + p_{-}(\yb, \yb'|\xb) \log\bigg(1 + \frac{p_{+}(\yb, \yb'|\xb)}{p_{-}(\yb, \yb'|\xb)}\bigg)   
d\yb d\yb' \\
&= 2\log2 + \int p_{+}(\yb, \yb'|\xb) \log\bigg(\frac{1/2 [p_{-}(\yb, \yb'|\xb) + p_{+}(\yb, \yb'|\xb)]}{p_{+}(\yb, \yb'|\xb)}\bigg) \\
&\qquad + p_{-}(\yb, \yb'|\xb) \log\bigg(\frac{1/2 [p_{-}(\yb, \yb'|\xb) + p_{+}(\yb, \yb'|\xb)]}{p_{-}(\yb, \yb'|\xb)}\bigg)   
d\yb d\yb' \\
&= 2\log2 - \mathrm{KL}\bigg(p_{+}\bigg\|\frac{p_{+} + p_{-}}{2}\bigg) - \mathrm{KL}\bigg(p_{-}\bigg\|\frac{p_{+} + p_{-}}{2}\bigg) \\
&= 2\log2 - 2 \cdot \mathrm{JSD}(p_{+}\|p_{-}),
\end{align*}
where the last equality is by the definition of JSD. This concludes the proof.
\end{proof}
Lemma~\ref{thm:closed-form solution} provides a closed-form solution to \eqref{eq:f-star} if we ignore the constraint set $\cF_{t}$. If this closed-form solution belongs to $\cF_{t}$, then it should also be the solution to \eqref{eq:f-star}. This observation is the key to the proof of Theorem~\ref{thm:main2}.
\begin{proof}[Proof of Theorem~\ref{thm:main2}] Under the condition of Theorem~\ref{thm:main2}, there exists a $p_{\btheta}$ such that 
\begin{align*}
p_{\btheta}(\yb|\xb) \propto p_{\btheta_{t}}(\yb|\xb)\big(p_{\mathrm{data}}(\yb|\xb)/p_{\btheta_{t}}(\yb|\xb) \big)^{1/\lambda}. 
\end{align*}
Therefore, there exists a function $\hat{Z}(\xb)$ such that 
\begin{align}
&p_{\btheta}(\yb|\xb) = \hat{Z}(\xb) \cdot p_{\btheta_{t}}(\yb|\xb)\big(p_{\mathrm{data}}(\yb|\xb)/p_{\btheta_{t}}(\yb|\xb) \big)^{1/\lambda}. \label{eq:finalmid1}
\end{align}
Applying logarithm function on both side of \eqref{eq:finalmid1} yields
\begin{align*}
\lambda \log(\hat{Z}(\xb)) + \log\bigg(\frac{p_{\mathrm{data}}(\yb|\xb)}{p_{\btheta_{t}}(\yb|\xb)}\bigg) = \lambda\log\bigg(\frac{p_{\btheta}(\yb|\xb)}{p_{\btheta_{t}}(\yb|\xb)}\bigg) \in \cF_{t}.    
\end{align*}
By Lemma~\ref{thm:closed-form solution}, $f^{*}(\xb,\yb) = \lambda \log(\hat{Z}(\xb)) + \log\big(\frac{p_{\mathrm{data}}(\yb|\xb)}{p_{\btheta_{t}}(\yb|\xb)}\big)$ is the global minimum of the following minimization problem, 
\begin{align}
\argmin_{f}\EE_{\yb\sim p_{\mathrm{data}}(\cdot|\xb), y'\sim p_{\btheta_t}(\cdot|\xb)}\big[ \ell\big(f(\xb, \yb) - f(\xb, \yb')\big) 
\big]. \label{eq:problem10}
\end{align}
Since $f^{*} \in \cF_{t}$, $f^{*}(\xb,\yb) = \lambda \log(\hat{Z}(\xb)) + \log\big(\frac{p_{\mathrm{data}}(\yb|\xb)}{p_{\btheta_{t}}(\yb|\xb)}\big)$ is also the global optimum of the optimization problem \eqref{eq:f-star}, 
\begin{align*}
\argmin_{f\in \cF_t}\EE_{\yb\sim p_{\mathrm{data}}(\cdot|\xb), y'\sim p_{\btheta_t}(\cdot|\xb)}\big[ \ell\big(f(\xb, \yb) - f(\xb, \yb')\big) 
\big].
\end{align*}
Therefore, we have proved that 
\begin{align}
&\min_{f}\EE_{\yb\sim p_{\mathrm{data}}(\cdot|\xb), y'\sim p_{\btheta_t}(\cdot|\xb)}\big[ \ell\big(f(\xb, \yb) - f(\xb, \yb')\big) 
\big] \notag \\
&= \min_{f\in \cF_t}\EE_{\yb\sim p_{\mathrm{data}}(\cdot|\xb), y'\sim p_{\btheta_t}(\cdot|\xb)}\big[ \ell\big(f(\xb, \yb) - f(\xb, \yb')\big)\big] \notag \\
&= \min_{\btheta \in \bTheta}L_{\method}(\btheta, \btheta_t)
. \label{eq:lastlast}
\end{align}
Since $\btheta_{t+1}$ is the global minimum of $L_{\method}(\btheta, \btheta_t)$. Then by \eqref{eq:lastlast}, $\lambda\log\Big(\frac{p_{\btheta_{t+1}}(\yb|\xb)}{p_{\btheta_{t}}(\yb|\xb)}\Big)$ should be the global minimum of problem \eqref{eq:problem10}. By Lemma~\ref{thm:closed-form solution}, there exists $Z(\xb)$ such that 
\begin{align*}
\lambda\log\bigg(\frac{p_{\btheta_{t+1}}(\yb|\xb)}{p_{\btheta_{t}}(\yb|\xb)}\bigg) = Z(\xb) + \log\bigg(\frac{p_{\mathrm{data}}(\yb|\xb)}{p_{\btheta_{t}}(\yb|\xb)}\bigg), 
\end{align*}
which leads to the result that $p_{\btheta_{t+1}}(\yb|\xb) \propto p_{\btheta_{t}}(\yb|\xb)\big(p_{\mathrm{data}}(\yb|\xb)/p_{\btheta_{t}}(\yb|\xb) \big)^{1/\lambda}$.

\end{proof}

\end{document}